\documentclass[lettersize,journal]{IEEEtran}
\usepackage[T1]{fontenc}
\usepackage{amsmath,amssymb,amsfonts,mathtools}
\usepackage{amsthm}
\usepackage{graphicx}
\usepackage{booktabs}
\usepackage[misc]{ifsym}
\usepackage{bm}
\usepackage{xcolor}
\usepackage{tabularx}
\usepackage{longtable}
\usepackage{enumitem}
\usepackage{array}
\usepackage{textcomp}
\usepackage{stfloats}
\usepackage{url}
\usepackage{verbatim}
\usepackage{cite}
\usepackage[ruled,linesnumbered]{algorithm2e}
\usepackage[hidelinks]{hyperref}

\newcommand{\Tr}{\operatorname{Tr}}

\newcommand{\NMSE}{\operatorname{NMSE}}
\newcommand{\NRMSE}{\operatorname{NRMSE}}

\newtheorem{theorem}{Theorem}
\newtheorem{proposition}{Proposition}
\newtheorem{lemma}{Lemma}
\newtheorem{corollary}{Corollary}
\newtheorem{remark}{Remark}

\begin{document}

\title{Free-Probability Kernels for Zero-Rollout Hyperparameter Selection in Reservoir Computing}

\author{%
\IEEEauthorblockN{Sara Malacarne$^1$, Andrea Ceni$^2$,
Claudio Gallicchio$^2$}\\
\IEEEauthorblockA{%
$^1$Telenor Research \& Innovation, Oslo, Norway\\
$^2$Department of Computer Science, University of Pisa,
Pisa, Italy}}

\maketitle

\begin{abstract}
Reservoir computing (RC) couples a fixed recurrent dynamical system with a trained lightweight readout, but this efficiency is partly lost during hyperparameter selection: the recurrent gain, input scale, and leakage rate determine the reservoir’s stability and temporal processing regime and are usually tuned through many rollouts. We introduce a deterministic, pilot-informed selector for leaky linear reservoirs followed by coordinate-wise nonlinear features. Free probability yields cross-lag propagation coefficients that summarize how the reservoir mixes past inputs. In the large-width limit, these coefficients define a deterministic temporal kernel that approximates the finite-reservoir feature geometry. Kernel ridge regression on a short labeled pilot sequence therefore ranks candidate operating regimes without instantiating or rolling out a reservoir, and the selected configuration transfers across widths. Across ten synthetic temporal benchmarks, zero-rollout selection obtains a mean deployment score of $0.772$, compared with $0.774$ for exhaustive simulation-based search, while avoiding $156\,600$ selection rollouts. With a small rollout budget, the proposed ranking provides the strongest mean performance at every tested budget and reaches the exhaustive reference using $4.8\%$ of its rollout cost. On four public electricity-transformer-temperature (ETT) forecasting datasets, five retained candidates recover the exhaustive operating point on three datasets. On multivariate cellular-traffic forecasting, 15 rollouts per cell reach the 462-rollout exhaustive reference and outperform random search and Bayesian optimization at low budgets. These results position free-probability kernels as deterministic surrogates for selecting reservoir operating regimes when validation rollouts are scarce.
\end{abstract}


\section{Introduction}
\label{sec:intro}

Reservoir computing (RC) provides an efficient framework for temporal
sequence learning: a fixed recurrent system maps the input history into
a high-dimensional representation, while only a lightweight readout is
trained~\cite{lukosevicius2009reservoir,Tanaka2019}. Its canonical
instance is the Echo State Network (ESN)~\cite{jaeger2001echo}.
Although readout training is inexpensive, performance depends strongly
on hyperparameters controlling the recurrent dynamics, particularly the
recurrent gain, input scale, and leakage rate. These parameters govern
stability, memory, and the relative influence of recent and distant
inputs, and their best values can vary substantially across tasks and
temporal scales.

In standard practice, these hyperparameters are selected by repeatedly
instantiating finite-width reservoirs, generating their state
trajectories, fitting readouts, and evaluating validation error.
Consequently, the search cost grows with the candidate grid, the
reservoir width, and the number of random realizations. The selected
configuration may also depend on the width used during search, so
selection may need to be repeated when the deployment width changes.

We address this bottleneck with a deterministic, task-informed selector
that requires no finite-reservoir rollout during candidate ranking.
Given a short labelled input--output sequence from the downstream task,
the method evaluates each candidate through a deterministic large-width
kernel. The pilot sequence describes the prediction task; it is not
obtained by running the candidate reservoirs. The resulting ranking can
be used directly as a zero-rollout selector or as a pre-screen that
reduces the candidate grid before a small number of task-direct
finite-reservoir evaluations.

We study a leaky linear recurrence followed by coordinate-wise
nonlinear readout features. For this architecture, the nonlinear
feature Gram matrix converges at large reservoir width to a
deterministic temporal kernel. Its covariance is governed by mixed
cross-lag propagation moments of the recurrent matrix, which we compute
using free probability (FP). For each operating point
$\theta=(\sigma_r,\sigma_{\rm in},\alpha)$, the resulting kernel
approximates the feature geometry of the finite reservoir with those
same parameters. Kernel ridge regression on the pilot sequence can
therefore rank the candidate operating points, after which the selected
parameters are transferred to the finite reservoir used for deployment.

Existing reservoir kernels are used mainly either to describe
large-width reservoir behaviour or as predictors in their own right.
Our use is different: for each candidate
$\theta=(\sigma_r,\sigma_{\rm in},\alpha)$, the deterministic kernel
approximates the feature Gram matrix of the finite reservoir with the
same parameters. We use this approximation to rank candidates on
labelled pilot data, and then deploy the selected $\theta^\star$ in the
finite reservoir.

The main analytical difficulty is that the same recurrent matrix is
applied repeatedly over time, so the effects of inputs arriving at
different lags are not independent. Our derivation captures these
cross-time dependencies and shows that, at large width, the resulting
nonlinear feature similarities converge to a deterministic kernel. To
the best of our knowledge, this is the first task-informed,
zero-rollout method that uses such a kernel to select the
hyperparameters of a finite reservoir.

The main contributions are as follows:
\begin{enumerate}[leftmargin=*,itemsep=0pt]
    \item
    We derive the deterministic large-width post-nonlinearity kernel
    of a leaky linear reservoir. The analysis computes the mixed
    cross-lag propagation moments generated by repeated use of the same
    recurrent matrix and establishes the coordinate-level
    self-averaging required by the nonlinear feature map. We also
    obtain a controlled complete-history extension and explicit
    kernels for several structured recurrent ensembles.

    \item
    We use this kernel as a task-informed surrogate for
    finite-reservoir hyperparameter selection. Candidate values of
    $(\sigma_r,\sigma_{\rm in},\alpha)$ are ranked from labelled pilot
    data without instantiating any candidate reservoir. If the
    deterministic validation score has a unique best candidate, then a
    sufficiently wide finite reservoir selects the same candidate with
    high probability.

    \item
    We evaluate zero-rollout selection and FP-based pre-screening on
    synthetic temporal benchmarks, the public ETT forecasting
    datasets~\cite{zhou2021informer}, and an operational
    cellular-traffic forecasting task. We compare against exhaustive
    task-direct selection, a memory-based proxy, nonlinear-ESN
    selection, random search~\cite{bergstra2012random}, and Bayesian
    optimization using the tree-structured Parzen estimator
    (TPE)~\cite{bergstra2011algorithms}, under matched
    finite-reservoir rollout budgets.
\end{enumerate}

Complete proofs, an explicit kernel catalogue for alternative recurrent
random matrices, and additional experimental results are provided in
Appendices~\ref{app:fixed-context-proofs}--\ref{sec:forecasting-horizons}.


\section{Background and Problem Formulation}
\label{sec:background}

Existing approaches to reservoir hyperparameter selection can be
divided broadly into task-based search and dynamical proxies.
Task-based methods evaluate candidate configurations according to their
downstream validation performance. The simplest examples are grid and
random search~\cite{bergstra2012random}, while more adaptive strategies
include Bayesian optimization for reservoir
parameters~\cite{maat2018efficient}, gradient-based
optimization~\cite{thiede2019gradient}, and evolutionary
methods~\cite{matzner2022hyperparameter}. More efficient validation
schemes reduce or reuse some of these evaluations
~\cite{racca2021robust,lukosevicius2019efficient}. Despite their
different search strategies, these methods share the same basic cost:
to assess a new candidate, they instantiate a finite reservoir,
generate its state trajectory, fit a readout, and evaluate its
predictions.

A different strategy is to select or adapt reservoirs using properties
of their internal dynamics. Echo-state-property conditions and
spectral-radius guidelines identify stable operating regions
~\cite{jaeger2001echo,yildiz2012revisiting,
lukosevicius2012practical}, while memory capacity and
Jacobian-based indicators quantify aspects of temporal processing
~\cite{verstraeten2009quantification}. Reservoir dynamics may also be
adapted without labels, for example through intrinsic
plasticity~\cite{schrauwen2008improving}. These criteria provide useful
information about stability, memory, or dynamical regime, but they are
not designed to rank candidates directly according to the labelled
downstream validation objective. When they depend on realized states,
they also still require a finite reservoir to be instantiated and run.

Large-width reservoir theory offers a third perspective: replacing the
random finite system by a deterministic limiting description.
Hermans and Schrauwen formulated infinite-width ESNs as recurrent
kernel machines~\cite{hermans2012recurrent}. Couillet et al.\ used
random-matrix theory to characterize the training and testing behaviour
of large linear ESNs~\cite{couillet2016random}, while Dong et al.\
derived nonlinear recurrent kernels and used them directly for
prediction~\cite{dong2020reservoir}. Related recurrent-kernel
constructions have since been developed for leaky, sparse, and deep
reservoir topologies~\cite{d2025comparison}. Gonon, Grigoryeva, and
Ortega~\cite{gonon2025reservoir} instead construct an
infinite-dimensional Volterra reservoir whose kernel defines a
universal temporal predictive model.

These works show that the behaviour of a large reservoir can often be
represented by a deterministic kernel. Their main purpose, however, is
to characterize the reservoir or to use the limiting kernel itself as
the predictive model. We address a different question: whether such a
kernel can replace finite-reservoir rollouts during hyperparameter
selection. For each candidate
$\theta=(\sigma_r,\sigma_{\rm in},\alpha)$, we construct the
large-width kernel corresponding to the finite reservoir with those
same parameters. Its performance on labelled pilot data is used to
rank the candidates, after which the selected
$\theta^\star$ is transferred back to the finite reservoir used for
deployment.

We now introduce the reservoir family for which this correspondence is derived. Consider the leaky linear recurrence
\begin{equation}
x_{t+1}
=
A x_t+\alpha W_{\rm in}u_t,
\qquad
A=(1-\alpha)I_n+\alpha W_r,
\label{eq:model}
\end{equation}
equipped with coordinate-wise nonlinear readout features
\begin{equation}
z_t=\psi(x_t),
\qquad
\hat y_t=W_{\rm out}^{\top}z_t.
\label{eq:nonlinearity}
\end{equation}
Here,
$u_t\in\mathbb{R}^{d_{\rm in}}$,
$x_t,z_t\in\mathbb{R}^{n}$, and
$\hat y_t,y_t\in\mathbb{R}^{d_{\rm out}}$ denote the input, linear
state, nonlinear feature vector, prediction, and target, respectively.
The recurrent and input matrices
$W_r\in\mathbb{R}^{n\times n}$ and
$W_{\rm in}\in\mathbb{R}^{n\times d_{\rm in}}$
are randomly initialized and remain fixed, while only
$W_{\rm out}\in\mathbb{R}^{n\times d_{\rm out}}$ is fitted.
The leakage rate is $\alpha\in(0,1]$, the reservoir is initialized at
$x_0=0$, and $\psi$ is applied coordinate-wise. The experiments use
$\psi=\tanh$. The scalings of $W_r$ and $W_{\rm in}$ introduce the
recurrent gain $\sigma_r$ and input scale $\sigma_{\rm in}$ and are
specified in Section~\ref{sec:method}.

Placing the nonlinearity after the recurrence separates temporal memory
from nonlinear processing. Related linear-transition designs also
appear in modern state-space sequence models such as S4 and
LRU~\cite{gu2021efficiently,orvieto2023resurrecting}. Linear recurrent
dynamics do not by themselves imply limited temporal modelling
capacity: stable linear reservoirs equipped with sufficiently expressive
polynomial or neural-network readouts form universal approximation
families for broad classes of fading-memory
filters~\cite{gonon2020universality,grigoryeva2018universal}. In such architectures, the linear
recurrence represents and propagates the input history, while the
readout supplies the nonlinear approximation capacity.

Our model uses the lightweight feature map
$z_t=\tanh(x_t)$ followed by a trained linear output layer. This
coordinate-wise readout is not itself covered verbatim by the broader
universality theorem, but it has recently been shown to be competitive
with fully nonlinear ESNs on a range of temporal
tasks~\cite{lagomarsini2025benchmarking}. More importantly for the
present analysis, placing the nonlinearity outside the recurrent loop
keeps the state linear in the propagated input history.

Indeed, unrolling~\eqref{eq:model} from $x_0=0$ gives
\begin{equation}
x_t
=
\alpha\sum_{k=0}^{t-1}A^kW_{\rm in}u_{t-1-k}.
\label{eq:unrolled-intro}
\end{equation}
Because the same recurrent matrix is applied repeatedly over time, the
covariance between states depends on the mixed normalized moments
\begin{equation}
\frac{1}{n}
\Tr\!\left(A^k(A^\ell)^\top\right).
\label{eq:mixed-trace-background}
\end{equation}
These moments form the reservoir’s cross-lag propagation profile: diagonal terms measure the average squared gain of delayed echoes, whereas off-diagonal terms measure their overlap. Unlike the spectral radius, this profile also retains cross-time coupling and non-normal amplification.

The next section derives the large-width limits of these mixed moments, uses them to construct the nonlinear feature kernel, and applies the kernel to zero-rollout candidate ranking.

\section{Free-Probability Kernel for Temporal Model Selection}
\label{sec:method}

Sections~\ref{sec:model} and~\ref{sec:kernel-notation} specify the
reservoir family and introduce the empirical-kernel notation. The new
theoretical results begin in Section~\ref{sec:fp-kernel}, where we
derive the deterministic large-width kernel at fixed context length;
Sections~\ref{sec:extensions} and~\ref{sec:selector} extend this
construction to the complete washed-out history and use it for
hyperparameter selection.

Throughout this section,
$\theta=(\sigma_r,\sigma_{\rm in},\alpha)$ denotes a fixed candidate
configuration. To avoid overloading the notation, we suppress the
dependence on $\theta$ in $A$, $W_r$, $W_{\rm in}$, and the reservoir
states, but retain it in the kernel and covariance quantities because
the selector compares these objects across candidates. The subscript
$n$ identifies a finite-width random kernel, while kernels without
$n$ are deterministic large-width limits. The superscript $(L)$
denotes an $L$-context quantity, and $(\infty)$ denotes complete
history.

\subsection{Leaky linear reservoir and stability}
\label{sec:model}

We use the leaky linear reservoir model of Section~\ref{sec:background}, eq.~\eqref{eq:model} with a scaled real Ginibre recurrent
matrix $W_r$ and an independent Gaussian input matrix $W_{\rm in}$:
\begin{equation}
\begin{aligned}
  W_r
  &= \frac{\sigma_r}{\sqrt{n}}\,G_r,
  & G_r &\in \mathbb{R}^{n\times n}, \\
  W_{\rm in}
  &= \frac{\sigma_{\rm in}}{\sqrt{d_{\rm in}}}\,G_{\rm in},
  & G_{\rm in} &\in \mathbb{R}^{n\times d_{\rm in}}, \\
  (G_r)_{ij}
  &\overset{\mathrm{i.i.d.}}{\sim} \mathcal{N}(0,1),
  & (G_{\rm in})_{ij}
  &\overset{\mathrm{i.i.d.}}{\sim} \mathcal{N}(0,1), \\
  G_r &\perp G_{\rm in}. &&
\end{aligned}
\label{eq:scaling}
\end{equation}
The $1/\sqrt{n}$ normalization keeps the spectrum and operator norm of
$W_r$ on an $O(1)$ scale as the reservoir width grows. The normalization
of $W_{\rm in}$ controls the scale of the injected input. Indeed, for
fixed deterministic vectors $u,v\in\mathbb{R}^{d_{\rm in}}$,
\begin{equation}
  \mathbb{E}_{W_{\rm in}}
  \!\left[W_{\rm in}u v^\top W_{\rm in}^\top\right]
  =
  \frac{\sigma_{\rm in}^2}{d_{\rm in}}
  (u^\top v)I_n .
  \label{eq:input-isotropy}
\end{equation}
In particular,
\begin{equation}
  \operatorname{Var}_{W_{\rm in}}
  \!\left[\left(W_{\rm in}u\right)_i\right]
  =
  \frac{\sigma_{\rm in}^2}{d_{\rm in}}\lVert u\rVert^2 .
  \label{eq:input-variance}
\end{equation}
Our asymptotic analysis takes $n\to\infty$ with $d_{\rm in}$ fixed; no
additional high-dimensional assumption on the input vectors is required.
The independence of $G_r$ and $G_{\rm in}$ allows us to condition on the
recurrent matrix while retaining an independent Gaussian input map.

For the linear recurrence, the echo-state property can be characterized
exactly. Let $x_t$ and $x_t'$ be two trajectories driven by the same
input sequence but initialized at $x_0$ and $x_0'$, respectively. Their
difference satisfies
\[
  x_t-x_t' = A^t(x_0-x_0').
\]
Hence, the influence of the initial condition vanishes for every pair
$x_0,x_0'$ if and only if $A^t \to 0$ as $t\to\infty$, which is
equivalent to
\begin{equation}
  \rho(A)<1.
  \label{eq:linear-esp}
\end{equation}

For $W_r=(\sigma_r/\sqrt{n})G_r$, the circular law places the
large-width eigenvalue cloud of $W_r$ in a disk centered at the origin
with radius $\sigma_r$~\cite{BordenaveChafai2012}. Therefore, the eigenvalue cloud of
\[
  A=(1-\alpha)I_n+\alpha W_r
\]
is asymptotically a disk centered at $1-\alpha$ with radius
$\alpha\sigma_r$. Its asymptotic spectral radius is consequently
\begin{equation}
  \rho_{\rm FP}(\alpha,\sigma_r)
  :=
  (1-\alpha)+\alpha\sigma_r.
  \label{eq:rho-fp}
\end{equation}
For $\alpha>0$, the asymptotic linear echo-state condition
$\rho_{\rm FP}<1$ is equivalent to $\sigma_r<1$. A finite-width
robustness margin is introduced in Section~\ref{sec:selector}.

\subsection{Readout and finite-reservoir kernel notation}
\label{sec:kernel-notation}

The matrices $W_r$ and $W_{\rm in}$ in
eq.~\eqref{eq:scaling} are never trained. The readout features are
\begin{equation}
  z_t=\psi(x_t),
  \qquad
  z_{t,i}=\psi(x_{t,i}),
  \quad i=1,\ldots,n,
  \label{eq:feature-map}
\end{equation}
where $\psi:\mathbb{R}\to\mathbb{R}$ is applied coordinate-wise. The
experiments use $\psi=\tanh$. The theoretical kernel result assumes
$\psi\in C_b^1(\mathbb{R})$, meaning that $\psi$ is bounded and
continuously differentiable with bounded derivative.

Let $\mathcal{T}$ and $\mathcal{V}$ denote the training and validation
time indices of the pilot sequence. Stacking the corresponding
untruncated feature vectors gives
\begin{equation}
  Z_T=\begin{pmatrix}z_t^\top\end{pmatrix}_{t\in\mathcal{T}},
  \qquad
  Z_V=\begin{pmatrix}z_t^\top\end{pmatrix}_{t\in\mathcal{V}}.
  \label{eq:feature-matrices}
\end{equation}
Since the feature map $\psi$ is fixed within each construction, we
suppress it in the notation for the pilot kernel blocks. The
width-normalized empirical kernel blocks for candidate $\theta$ are
\begin{equation}
\begin{aligned}
  K_{n,\theta,TT}
  &=\frac{1}{n}Z_TZ_T^\top, \\
  K_{n,\theta,VT}
  &=\frac{1}{n}Z_VZ_T^\top .
\end{aligned}
\label{eq:empirical-kernel-blocks}
\end{equation}
They are indexed by pilot time steps, not by reservoir coordinates.
Writing $Y_T$ for the training targets, kernel ridge regression predicts
on the validation block as
\begin{equation}
  \widehat{Y}_{V,n}(\theta,\lambda)
  =K_{n,\theta,VT}
  \left(K_{n,\theta,TT}+\lambda I\right)^{-1}Y_T .
  \label{eq:kernel-ridge}
\end{equation}

The blocks in~\eqref{eq:empirical-kernel-blocks} are computed from the
untruncated finite-reservoir trajectory. Section~\ref{sec:fp-kernel}
first derives deterministic large-width limits for their
$L$-context counterparts. Section~\ref{sec:extensions} then shows that,
in the stable regime, the finite-context kernels converge to the
complete-history kernel associated with the washed-out recurrence.

\subsection{Finite-context free-probabilistic kernel construction}
\label{sec:fp-kernel}

Starting from $x_0=0$, the linear recurrence can be unrolled exactly:
\begin{equation}
  x_t
  =\alpha\sum_{k=0}^{t-1}
  A^kW_{\rm in}u_{t-1-k} .
  \label{eq:unrolled-state}
\end{equation}
For a fixed context length $L<\infty$, the selector retains the most
recent $L+1$ propagated inputs and uses the truncated state
\begin{equation}
  x_t^{(L)}
  :=\alpha\sum_{k=0}^{L}
  A^kW_{\rm in}u_{t-1-k} .
  \label{eq:truncated-state}
\end{equation}
The corresponding empirical nonlinear kernel entry at width $n$ is
\begin{equation}
  K_{\psi,n,\theta}^{(L)}(t,s)
  :=
  \frac{1}{n}\,
  \psi(x_t^{(L)})^\top
  \psi(x_s^{(L)}).
  \label{eq:empirical-nonlinear-kernel}
\end{equation}

For the pilot train and validation sets, define the corresponding
empirical $L$-context blocks entrywise by
\begin{equation}
\begin{aligned}
  \bigl[K_{n,\theta,TT}^{(L)}\bigr]_{t,s}
  &:=K_{\psi,n,\theta}^{(L)}(t,s),
  &&t,s\in\mathcal{T}, \\
  \bigl[K_{n,\theta,VT}^{(L)}\bigr]_{t,s}
  &:=K_{\psi,n,\theta}^{(L)}(t,s),
  &&t\in\mathcal{V},\ s\in\mathcal{T}.
\end{aligned}
\label{eq:empirical-truncated-kernel-blocks}
\end{equation}

Because the recurrence is linear, the state covariance is determined by
mixed normalized traces of propagated recurrent matrices. For fixed
$k,\ell\geq0$, define
\begin{equation}
  \tau_{k,\ell}
  :=\sum_{j=0}^{\min(k,\ell)}
  \binom{k}{j}\binom{\ell}{j}
  (1-\alpha)^{k+\ell-2j}(\alpha\sigma_r)^{2j}.
\label{eq:tau}
\end{equation}
These coefficients quantify how input components injected at two lags
remain correlated after propagation through the reservoir.

\begin{proposition}[Ginibre cross-lag propagation coefficients]
\label{prop:moments}
Let
\[
  A_n=(1-\alpha)I_n+\alpha W_r^{(n)},
  \qquad
  W_r^{(n)}=\frac{\sigma_r}{\sqrt{n}}G_r^{(n)}.
\]
For every fixed $k,\ell\geq0$,
\begin{equation}
  \frac{1}{n}\operatorname{Tr}
  \!\left(A_n^k(A_n^\ell)^\top\right)
  \xrightarrow{L^1}
  \tau_{k,\ell}.
  \label{eq:trace-moment-limit}
\end{equation}
\end{proposition}

The proof expands the mixed powers by the binomial theorem, uses the
circular-element limit for the normalized $*$-moments, and controls
fluctuations by Gaussian concentration. Complete details are given in
Appendix~\ref{app:fixed-context-proofs}.

The trace limit controls only an average over coordinates. Because the
activation is applied coordinate-wise, we additionally require the
diagonal propagation terms to self-average.
\begin{lemma}[Averaged diagonal self-averaging]
\label{lem:diag}
Let $A_n=(1-\alpha)I_n+\alpha W_r^{(n)}$. For every fixed
$L<\infty$,
\begin{equation}
\begin{aligned}
  \max_{0\leq k,\ell\leq L}
  \frac{1}{n}\sum_{i=1}^n
  \Bigl|
  \bigl[A_n^k(A_n^\ell)^\top\bigr]_{ii}
  -\tau_{k,\ell}
  \Bigr|^2
  \xrightarrow{\mathbb{P}}0.
\end{aligned}
\label{eq:diag-avg}
\end{equation}
\end{lemma}
Thus, although individual reservoir units remain random, their 
diagonal propagation profile becomes deterministic. This coordinate-level result is essential because $\psi$ acts neuronwise; dependence among different units is handled by the concentration step below.
The complete proof is given in Appendix~\ref{app:fixed-context-proofs}.

Using the cross-lag propagation
coefficients, define the deterministic
linear-state covariance
\begin{equation}
\begin{aligned}
  Q_{\theta}^{(L)}(t,s)
  :=\alpha^2\frac{\sigma_{\rm in}^2}{d_{\rm in}}
  \sum_{k,\ell=0}^{L}
  \tau_{k,\ell}\,
  u_{t-1-k}^\top u_{s-1-\ell},
\end{aligned}
\label{eq:q}
\end{equation}
where $\theta=(\sigma_r,\sigma_{\rm in},\alpha)$, and the associated
$2\times2$ covariance matrix
\begin{equation}
  \Sigma_{\theta}^{(L)}(t,s)
  :=
  \begin{pmatrix}
    Q_{\theta}^{(L)}(t,t) & Q_{\theta}^{(L)}(t,s)\\
    Q_{\theta}^{(L)}(t,s) & Q_{\theta}^{(L)}(s,s)
  \end{pmatrix}.
  \label{eq:Sigma-limit}
\end{equation}

The nonlinear feature kernel is obtained by passing the limiting
Gaussian state covariance through the coordinate-wise activation.
\begin{theorem}[Deterministic nonlinear reservoir kernel]
\label{thm:kernel}
Let $\psi\in C_b^1(\mathbb{R})$. For fixed $L,t,s$ and a deterministic
bounded input sequence $(u_r)_r$, as $n\to\infty$,
\begin{equation}
  K_{\psi,n,\theta}^{(L)}(t,s)
  \xrightarrow{\mathbb{P}}
  K_{\psi,\theta}^{(L)}(t,s),
  \label{eq:kernel-limit}
\end{equation}
where
\begin{equation}
\begin{aligned}
  K_{\psi,\theta}^{(L)}(t,s)
  &:=\mathbb{E}\!\left[\psi(g_t)\psi(g_s)\right], \\
  (g_t,g_s)
  &\sim\mathcal{N}\!\left(
    0,\Sigma_{\theta}^{(L)}(t,s)
  \right).
\end{aligned}
\label{eq:Gaussian-kernel}
\end{equation}
The convergence is in probability with respect to the joint randomness
of $W_r$ and $W_{\rm in}$.
\end{theorem}
\paragraph{RC interpretation and proof idea} 
Conditional on the recurrent matrix, each reservoir unit receives a Gaussian projection of the same propagated input history. Lemma~\ref{lem:diag} shows that the unitwise covariance profiles self-average, and concentration over the input weights removes the remaining finite-width fluctuations. Hence a wide random reservoir induces a deterministic temporal feature geometry at kernel level. On any fixed pilot set, the empirical Gram blocks converge in Frobenius norm to their deterministic counterparts. Complete details are given in Appendix~\ref{app:fixed-context-proofs}.

The nonlinear expectation in~\eqref{eq:Gaussian-kernel} can be
evaluated for the exact feature map $\psi=\tanh$ by deterministic
low-dimensional Gaussian quadrature. For each pair $(t,s)$, the
covariance matrix $\Sigma_\theta^{(L)}(t,s)$ defines a bivariate
Gaussian law, and $K_{\tanh,\theta}^{(L)}(t,s)$ is obtained by
Gauss--Hermite quadrature of
$\mathbb{E}[\tanh(g_t)\tanh(g_s)]$. Since
Theorem~\ref{thm:kernel} holds for any $\psi\in C_b^1$, a closed-form
erf/arcsine surrogate is also available; its uniform entrywise
approximation bound and empirical validation are given in
Remark~\ref{rem:erf-surrogate} and Appendix~\ref{app:surrogate-validation}. The experimental use of the exact
kernel and of this surrogate is specified in Section~\ref{sec:experiments}.

\subsection{Complete history and other recurrent ensembles}
\label{sec:extensions}

A finite context is the implementable selector; the complete-history limit connects it to the fading-memory interpretation of RC. When $\sigma_r<1$, the reservoir forgets remote inputs geometrically, so increasing $L$ recovers the washed-out state with controlled error.
For a bounded input history and $\sigma_r<1$,
define
\begin{equation}
  Q_{\theta}^{(\infty)}(t,s)
  :=
  \alpha^2\frac{\sigma_{\rm in}^2}{d_{\rm in}}
  \sum_{k,\ell=0}^{\infty}
  \tau_{k,\ell}\,
  u_{t-1-k}^{\top}u_{s-1-\ell},
  \label{eq:q-infinite}
\end{equation}
and let $K_{\psi,\theta}^{(\infty)}(t,s)$ be the Gaussian nonlinear
kernel obtained from this covariance as in
\eqref{eq:Gaussian-kernel}. 
The covariance series converges absolutely,
and the finite-context kernels converge geometrically to the
complete-history kernel:
\begin{equation}
  K_{\psi,\theta}^{(L)}(t,s)
  \longrightarrow
  K_{\psi,\theta}^{(\infty)}(t,s),
  \quad L\to\infty.
  \label{eq:kernel-context-limit}
\end{equation}
Because the admissible candidate grid is finite, this convergence is
uniform over the candidates used by the selector.

For a complete bounded input history, let
\begin{equation}
  x_t^{(\infty)}
  :=
  \alpha\sum_{k=0}^{\infty}
  A^k W_{\rm in}u_{t-1-k}
  \label{eq:complete-washed-out-state}
\end{equation}
denote the complete washed-out state. For real Ginibre recurrence with
$\sigma_r<1$, this series is well defined almost surely for all
sufficiently large $n$. Strong convergence controls the powers of the
non-normal transition matrix and hence the finite-width truncation
tail. Consequently, on every fixed finite pilot set,
\begin{equation}
  K_{\psi,n,\theta}^{(\infty)}(t,s)
  :=
  \frac{1}{n}
  \psi\!(x_t^{(\infty)})^\top
  \psi\!(x_s^{(\infty)})
  \xrightarrow[n\to\infty]{\mathbb{P}}
  K_{\psi,\theta}^{(\infty)}(t,s).
  \label{eq:full-history-kernel-limit}
\end{equation}
The quantitative truncation bounds and the full Ginibre argument are
given in Appendix~\ref{app:full-history-proof}.

The same construction applies to recurrent ensembles with deterministic
mixed propagation moments and diagonal self-averaging. The
Haar-orthogonal argument and explicit kernels for cyclic, circulant,
skew-symmetric, and complex-valued recurrences are given in
Appendices~\ref{app:haar-diagonal-proof}--\ref{app:recurrent-matrix-kernels}. The experiments in this paper use real
Ginibre recurrence throughout.

\subsection{Zero-rollout selection on temporal pilot data}
\label{sec:selector}

The kernel now acts as a virtual reservoir during model selection. For each operating point, it captures the large-width temporal feature geometry of the corresponding finite reservoir without constructing recurrent weights or state trajectories.

For each candidate $\theta=(\sigma_r,\sigma_{\rm in},\alpha)$, the
construction above replaces the empirical feature kernel of a large
random reservoir by a deterministic kernel computed directly from a
short labelled pilot sequence. Let $K_{\mathrm{sel},\theta}^{(L)}$ denote the deterministic
selection kernel used in a given experiment. In the experiments below,
this is either the exact $\tanh$ kernel evaluated by quadrature or the
closed-form erf surrogate described above; the precise choice is stated
in Section~\ref{sec:experiments}.
For pilot train and validation index sets $\mathcal{T}$ and
$\mathcal{V}$, define the deterministic kernel blocks entrywise by
\begin{equation}
\begin{aligned}
  \bigl[K_{\theta,TT}^{(L)}\bigr]_{t,s}
  &:=K_{{\rm sel},\theta}^{(L)}(t,s),
  && t,s\in\mathcal{T}, \\
  \bigl[K_{\theta,VT}^{(L)}\bigr]_{t,s}
  &:=K_{{\rm sel},\theta}^{(L)}(t,s),
  && t\in\mathcal{V},\ s\in\mathcal{T},
\end{aligned}
\label{eq:deterministic-kernel-blocks}
\end{equation}
where each entry is computed from $Q_\theta^{(L)}$ by either exact tanh
quadrature or the closed-form erf surrogate, as specified above. For a
ridge parameter $\lambda$, the corresponding validation prediction is
\begin{equation}
  \widehat{Y}_V(\theta,\lambda)
  =K_{\theta,VT}^{(L)}
  \left(K_{\theta,TT}^{(L)}+\lambda I\right)^{-1}Y_T.
  \label{eq:deterministic-kernel-ridge}
\end{equation}

\paragraph{FP selector}

\begin{algorithm}[h]
\small
\DontPrintSemicolon
\KwIn{Pilot sequence $(u_t,y_t)$, split into train/validation;
context length $L$; candidate grid $\Theta$; ridge grid $\Lambda$}
\KwOut{Selected hyperparameters $\theta^\star$ and ranking regularizer
$\lambda^\star_{\rm sel}$}
\BlankLine
\For{each $\theta=(\sigma_r,\sigma_{\rm in},\alpha)\in\Theta$
satisfying the stability guard}{
Compute $\tau_{k,\ell}(\theta)$ via~\eqref{eq:tau}
for $0\leq k,\ell\leq L$;
Build $Q_{\theta}^{(L)}$ via~\eqref{eq:q};
compute $K_{{\rm sel},\theta}^{(L)}$;
form $K_{\theta,TT}^{(L)}$ and $K_{\theta,VT}^{(L)}$
via~\eqref{eq:deterministic-kernel-blocks};
\For{each $\lambda\in\Lambda$}{
Compute $\widehat{Y}_V(\theta,\lambda)$
via~\eqref{eq:deterministic-kernel-ridge};
compute $\NMSE_{\rm val}^{(L)}(\theta,\lambda)$;
}
}
\Return $(\theta^\star,\lambda^\star_{\rm sel})
=\arg\min_{\theta,\lambda}\NMSE_{\rm val}^{(L)}(\theta,\lambda)$;
\caption{Zero-rollout FP hyperparameter selector}
\label{alg:fp-selector}
\end{algorithm}

Algorithm~\ref{alg:fp-selector} presents the holdout-ridge version used
for the synthetic benchmarks. For ETT and Telco, the inner selection of
$\lambda$ is instead performed by generalized cross-validation on the
training block, after which the candidate is scored on the chronological
validation block. The reservoir-candidate ranking procedure is otherwise
unchanged.

All matrices in Algorithm~\ref{alg:fp-selector} are indexed by pilot
time steps rather than reservoir coordinates: no $n\times n$ recurrent
matrix is formed and no finite reservoir is rolled out. Here
$\NMSE_{\rm val}^{(L)}(\theta,\lambda)
=\sum_{t\in\mathcal{V}}\lVert \widehat{y}_t-y_t\rVert^2
/\sum_{t\in\mathcal{V}}\lVert y_t-\bar{y}_{\mathcal{V}}\rVert^2$
is the normalized mean-squared error on the pilot validation block,
with $\bar{y}_{\mathcal{V}}$ the validation target mean. After precomputing input inner products, each candidate kernel
requires $O(T_{\rm pilot}^2L^2)$ lag-pair operations and
$O(T_{\rm pilot}^2)$ memory, followed by kernel ridge regression on the
$T_{\rm train}\times T_{\rm train}$ block. Crucially, none of these
costs depends on the deployment width $n$. The regularizer
$\lambda^\star_{\rm sel}$ is used only for ranking; at deployment,
$\lambda$ is re-selected on the deployment training block over the same
canonical grid, using the kernel-scale feature convention of
Section~\ref{sec:experiments}, so that the selection-stage
and deployment-stage regularizers are directly comparable.

Theorem~\ref{thm:kernel} concerns kernel entries; the selector,
however, returns an $\arg\min$ over a finite grid. The following
corollary records the corresponding decision-level result.

\begin{corollary}[Consistency of FP selection]
\label{cor:selection-consistency}
Assume that the deterministic kernel and the finite reservoir use the
same feature map. Fix finite candidate and ridge grids with strictly
positive ridge parameters. For any fixed context length $L$, if the
deterministic validation score has a unique minimizer, then finite-width
reservoir selection returns the same minimizer with probability tending
to one as $n\to\infty$.

In the stable regime, a unique complete-history optimum is preserved by
all sufficiently large finite contexts, and selection based on the
complete finite-reservoir trajectory converges to this optimum as
$n\to\infty$.
\end{corollary}

The result follows from kernel convergence, continuity of kernel ridge
regression, and stability of a unique minimizer on a finite grid. The
complete statement and proof are given in Appendices~\ref{app:fixed-context-proofs}--\ref{app:full-history-proof}.

The corollary concerns agreement on the pilot validation objective; it
does not assert that the selected configuration minimizes deployment
test error. When several candidates are nearly tied, agreement may
require a larger context length or reservoir width. In the synthetic
experiments, the selector uses the erf surrogate while the finite
reservoir uses $\tanh$, so the corollary does not apply verbatim. The
selection-level impact of this approximation is evaluated separately in
Remark~\ref{rem:erf-surrogate} and Appendix~\ref{app:surrogate-validation}.

\paragraph{Stability guard}
Equation~\eqref{eq:rho-fp} gives the asymptotic spectral radius of the
Ginibre transition matrix. Since the large-width selector cannot detect
realization-specific instabilities of finite Ginibre matrices close to
the stability boundary, we apply the fixed conservative guard
\begin{equation}
  \rho_{\rm FP}(\alpha,\sigma_r)
  \leq 1-\delta,
  \qquad
  \delta=0.05.
  \label{eq:stability-guard}
\end{equation}
The margin is a practical finite-width safeguard rather than a claim of
optimality, and it is kept fixed across deployment widths. We apply the
guard to every candidate grid before selection. For recurrent matrices
whose finite-width spectral boundary is controlled exactly, such as
scaled Haar-orthogonal or cyclic-shift recurrence, the additional margin
is not required.

\section{Experiments}
\label{sec:experiments}

\subsection{Setup}

We evaluate the selectors in three complementary regimes. First, a
suite of ten synthetic sequence benchmarks covers nonlinear temporal
processing, explicit memory, delayed nonlinear transformations, and
forecasting of chaotic dynamical systems. Second, four public ETT-small
datasets (ETTh1, ETTh2, ETTm1, and ETTm2) test open-data real-world
forecasting at hourly and 15-minute sampling rates. Third, a proprietary cellular-traffic dataset tests operational
multivariate forecasting on ten randomly selected network cells. Thus,
the evaluation spans controlled mechanistic tasks, reproducible public
forecasting data, and a cell-replicated operational application. 

\textit{Synthetic benchmarks:}
The synthetic suite contains NARMA-$k$ for
$k\in\{10,20,30,50\}$~\cite{atiya2000new}, linear memory
capacity (MC)~\cite{jaeger2001memory}, the nonlinear
delayed-input task of Inubushi and
Yoshimura~\cite{inubushi2017reservoir}, and four forecasting
tasks generated from chaotic dynamical systems:
Lorenz63~\cite{lorentz1963deterministic}, Mackey--Glass with two
delay/horizon settings~\cite{mackey1977oscillation}, and
5D Lorenz96~\cite{lorenz1996predictability}.
Their deployment protocols and task parameters are
summarised in Table~\ref{tab:tasks}.

\begin{table}[t]
\centering
\scriptsize
\setlength{\tabcolsep}{3pt}
\caption{Synthetic task protocols. Deployment train/test lengths are
counted after washout. Forecast horizons are given in task time steps.}
\label{tab:tasks}
\begin{tabularx}{\columnwidth}{lrrrX}
\toprule
Task & Washout & Deploy train & Test & Input, target, and task parameter\\
\midrule
NARMA10 & 100 & 800 & 800 &
$u_t\!\sim\! U(0,0.5)$; NARMA target, order $k=10$\\
NARMA20 & 100 & 1200 & 1200 &
$u_t\!\sim\! U(0,0.5)$; NARMA target, order $k=20$\\
NARMA30 & 100 & 1200 & 1200 &
$u_t\!\sim\! U(0,0.5)$; NARMA target, order $k=30$\\
NARMA50 & 100 & 1600 & 1600 &
$u_t\!\sim\! U(0,0.5)$; NARMA target, order $k=50$\\
MC & 100 & 2000 & 2000 &
$u_t\!\sim\! U(-1,1)$; delayed targets at lags $1{:}200$\\
Inubushi & 100 & 1000 & 1000 &
$u_t\!\sim\! U(-1,1)$; $y_t=\sin(2u_{t-5})$\\
Lorenz63 & 200 & 2000 & 2000 &
scalar Lorenz63 forecast, $(\sigma,\rho,\beta)=(10,28,8/3)$, horizon $h=25$\\
SF-MG30 & 200 & 2000 & 2000 &
Mackey--Glass forecast, $\tau=30$, horizon $h=10$\\
MG84 & 200 & 1000 & 1000 &
Mackey--Glass forecast, $\tau=17$, horizon $h=84$\\
Lorenz96 & 200 & 2000 & 2000 &
5D Lorenz96 forecast, forcing $F=8$, horizon $h=25$\\
\bottomrule
\end{tabularx}
\end{table}

\paragraph{Real-data benchmarks}
We use the public ETT-small benchmark introduced
in~\cite{zhou2021informer}, comprising ETTh1, ETTh2,
ETTm1, and ETTm2.
For ETT, the input consists of oil temperature (OT)
and calendar covariates, and the target is future OT at
physical horizons of 1, 6, and 12 hours. ETTh1/ETTh2 are sampled hourly, whereas ETTm1/ETTm2 are
sampled every 15 minutes. The proprietary Telco benchmark uses hourly
voice-traffic key performance indicators (KPIs) together with calendar
covariates, and jointly predicts both KPIs over horizons $1{:}12$ for
each of ten cells. All real-data splits are chronological.
Section~\ref{sec:real-data} gives the context lengths, split sizes,
anchor strides, and rollout details.

\paragraph{Zero-rollout and exhaustive selectors}
\textbf{FP} is the proposed zero-rollout selector
(Algorithm~\ref{alg:fp-selector}). It ranks candidate operating points
by deterministic kernel NMSE on labelled task-specific pilot data,
without instantiating a finite reservoir. In the synthetic suite, three
independently generated pilot sequences are used and each candidate is
assigned its largest validation NMSE across the three pilots; this
conservative aggregation penalizes pilot-sensitive configurations. The
three pilot sequences play, for the deterministic FP selector, the role
that reservoir seeds play for the simulation-based selectors: FP is
replicated over independently sampled task pilots, whereas the empirical
selectors are replicated over finite-reservoir realizations.
\textbf{Memory}$_n$ is a task-agnostic empirical proxy. At each grid
point and selection width $n$, a ridge readout reconstructs
$D=2n$ delayed samples of a generic scalar input $u_t\sim U(-1,1)$
from the reservoir features. Its score follows Jaeger's memory-capacity
criterion~\cite{jaeger2001memory}:
\begin{equation}
  \mathrm{MemScore}(\theta) = \frac{1}{D}\sum_{d=1}^{D}
  \max\!\left(0,\;1 - \frac{\sum_t(u_{t-d}-\hat{u}_{t-d})^2}
  {\sum_t(u_{t-d}-\bar{u}_d)^2}\right),
  \label{eq:memory}
\end{equation}
where $\bar{u}_d$ is the sample mean of the lag-$d$ target. Scores are
averaged over three reservoir seeds. \textbf{Direct}$_n$ evaluates the
downstream validation task at width $n$ and ranks candidates by their
mean score over three reservoir seeds. \textbf{ESN Direct}$_n$ uses the
same task-direct protocol with the standard nonlinear ESN recurrence
$x_{t+1}=(1-\alpha)x_t+\alpha\tanh(W_rx_t+W_{\rm in}u_t)$.

\paragraph{Budgeted selectors}
We additionally compare three selectors under a matched retained
candidate budget $K$. \textbf{FP-K} first ranks the admissible grid with
FP at zero rollout cost and then applies Direct$_{500}$ only to the top
$K$ candidates; we write FP-$5$, FP-$10$, etc.\ for specific budgets. \textbf{Random-K} samples $K$ admissible candidates
uniformly without replacement and returns the candidate with the best
Direct$_{500}$ validation score. \textbf{TPE-K}, a sequential Bayesian-optimization baseline,
uses the outcomes of previous trials to direct later evaluations toward
promising regions of the search space. Specifically, it applies the
tree-structured Parzen estimator (TPE)~\cite{bergstra2011algorithms}
to the discrete $(\sigma_r,\sigma_{\rm in},\alpha)$ search axes for
$K$ task-direct trials and returns the best observed candidate. 

Each task-direct trial
uses three selection seeds, so one run of any budgeted selector spends
$3K$ finite-reservoir rollouts; FP's deterministic pre-ranking adds
none. Random-K and TPE-K are repeated for 20 independent selection draws on
the synthetic suite and 10 draws on ETT and Telco, reflecting the higher
cost of the forecasting evaluations. Each individual run still receives
the same matched budget of $K$ candidates; repetitions are used only to
estimate stochastic-selector variability.

\paragraph{Regularization protocol}
All ridge regressions in this paper (FP kernel selection, empirical
selection, and deployment) use a single canonical regularization grid
$\lambda\in\{10^{-12},10^{-11},\ldots,10^2\}$ expressed on the kernel
scale. Finite-reservoir readouts are fitted on width-normalized features
$z_t/\sqrt{n}$, so that the empirical Gram matrix $\tfrac{1}{n}ZZ^\top$
and the deterministic FP kernel live on the same scale and a given
$\lambda$ corresponds to the same effective regularization at every
stage and every reservoir width. The selection rule for $\lambda$
differs by regime. On the synthetic benchmarks, where the
data-generating processes are stationary and fresh i.i.d.-driven
sequences can be generated, $\lambda$ is selected by the holdout-ridge
protocol on this grid. On the real-data ETT and Telco benchmarks,
$\lambda$ is selected by generalized cross-validation (RidgeCV) on the
corresponding training block over the same grid, both inside
finite-reservoir selection rollouts and at deployment. This uses the
full training block for regularization selection, avoiding an
additional dependence on a single chronological validation window in the
nonstationary real-data setting.

\paragraph{FP kernel implementation}
All finite-reservoir deployments use coordinate-wise $\tanh$ features.
For FP selection, we use the exact $\tanh$ kernel, evaluated by
Gauss--Hermite quadrature, on all real-data benchmarks (ETT and
Telco). For the synthetic benchmark suite and the synthetic timing
sweeps, we use the closed-form erf surrogate from Remark~\ref{rem:erf-surrogate}.
This surrogate is used only as a computational accelerator in the
controlled synthetic setting; Appendix~\ref{app:surrogate-validation} verifies that matched
erf-surrogate and exact-$\tanh$ selections have negligible
selection-level impact on the checked synthetic settings.

\paragraph{Synthetic benchmark protocol}
The synthetic Cartesian grid contains $6\,480$ raw points. Applying the
fixed stability guard~\eqref{eq:stability-guard} leaves $5\,220$
admissible candidates. Full-grid empirical selection at
$n_{\rm select}=500$ therefore uses
$5220\times3=15\,660$ finite-reservoir rollouts per task, or
$156\,600$ across the ten-task suite. FP ranks the same admissible grid
with zero finite-reservoir rollouts. Its pilot sequences are generated
after the task-specific warmup and have length $T_{\rm pilot}=500$,
split into 333 training and 167 validation time steps; the deterministic
selection kernel uses context $L=50$. The washout lengths in
Table~\ref{tab:tasks} apply only to finite-reservoir deployment, since
FP selection does not instantiate or roll out a reservoir.

After selection, $\theta^\star$ is fixed and evaluated with new
reservoir seeds. Intermediate-width trends use 3 deployment seeds at
$n\in\{1\,000,3\,000,5\,000,10\,000\}$; the $n=20\,000$ synthetic
comparison uses 10 deployment seeds. In the posthoc synthetic budget
sweep, selected configurations use 3 deployment seeds.

\paragraph{Real-data benchmark protocol}
\label{par:real-data_protocol}
All forecasting experiments use physical horizons of 1, 6, and 12
hours. A \emph{forecast anchor} is a time index at which the preceding
input history is used to predict the future target vector, and hence
defines one forecasting example. Because nearby anchors in the long ETT
series produce strongly overlapping examples, we retain every $s$th
valid anchor, using the same anchor indices for all selectors.

For each anchor, the finite-reservoir state is obtained by a local
rollout over the input history represented by the FP kernel. In the FP
construction, $L$ is the maximum retained lag, so lags $0,\ldots,L$
are included and each state uses $L+1$ input samples. Thus, $L=96$
corresponds to 97 hourly samples spanning a maximum lag of 96 hours,
whereas $L=384$ corresponds to 385 quarter-hourly samples over the
same maximum-lag span. The FP kernel is evaluated only between the
selected forecast anchors.

For ETT, the input contains oil temperature (OT) and calendar
covariates, and the target is future OT. ETTh is sampled hourly and
uses $L=96$, horizons $h\in\{1,6,12\}$, and anchor stride 43.
ETTm is sampled every 15 minutes and uses $L=384$, horizons
$h\in\{4,24,48\}$, and anchor stride 173. Hence both variants use a
96-hour maximum retained lag and the same physical forecast horizons.
ETTh uses $8\,640/2\,880/2\,880$ train/validation/test rows, while
ETTm uses $34\,560/11\,520/11\,520$ rows.

For the cellular-traffic task, the input combines data and voice KPIs
with sine--cosine calendar covariates, and the target is the joint
forecast of both KPIs over horizons $1{:}12$. Each cell uses 841
training, 360 validation, and 932 test rows, with $L=96$. Because the
Telco series are shorter than ETT, every valid anchor is retained.

In all forecasting experiments, one readout jointly predicts all
horizons from a single reservoir state. Accordingly, one configuration
$\theta^\star$ is selected per ETT dataset or per Telco cell using the
validation score aggregated over all horizons. Every selector,
including full-grid Direct$_{500}$, optimizes this same aggregate
objective. Selecting horizon-specific configurations is also possible
by restricting the targets to one horizon, but is not considered here.

The ETT and Telco Cartesian grids each contain 245 raw candidates, of
which 154 satisfy the stability guard. All selectors operate on this
same admissible set, so exhaustive empirical selection requires
$154\times3=462$ finite-reservoir rollouts per ETT dataset and per
Telco cell.

For FP selection, the chronological training and validation blocks
define the \emph{pilot split}: the deterministic kernel model is fitted
on the training anchors and evaluated on the validation anchors. Since
each real-data series is fixed, a single pilot split is used rather than
averaging over independently generated pilot sequences.

All real-data models are deployed at width $n=20\,000$ with three fresh
reservoir seeds. For ETT, reported deviations are computed over
deployment seeds for deterministic FP-K and over independent selection
draws for Random-K and TPE-K, after averaging deployment seeds within
each draw. For Telco, results are first averaged over selection draws
within each cell and then summarized by the mean and standard deviation
across the 10 cells.

\paragraph{Metrics and statistical comparisons}
Performance is reported as $1-\NRMSE$; higher is better. 
The source of averaging and variability (tasks, deployment seeds, selection draws, or cells) is specified in the corresponding experimental protocol.
Paired selector comparisons use Wilcoxon
signed-rank tests over tasks or cells, with Holm correction across the
reported pairwise comparisons. Stochastic budget repetitions are used to
characterise selection variability, not as additional independent tasks.

\subsection{Synthetic benchmark results}
\label{sec:synthetic-results}

The synthetic suite evaluates three complementary uses of the FP
selector: zero-rollout selection when no simulation budget is available,
width-portable deployment as the reservoir grows, and FP-based
pre-screening when a small rollout budget can be spent. The paragraphs
below address these in turn.

\paragraph{Zero-rollout selection}
Direct$_{500}$ attains the highest mean deployment score across the
ten-task suite, reaching $0.774$ at $n=20\,000$
(Table~\ref{tab:avg-results}). This is expected: it directly evaluates
finite-reservoir downstream validation performance during selection.
FP reaches $0.772$ while using zero finite-reservoir search rollouts.
The task-level comparison in Figure~\ref{fig:bars} shows that FP is
best or tied-best on MC, NARMA30, NARMA50, and Lorenz96, and remains
close to the best selector on several other tasks. The largest gaps
occur on Lorenz63 and Inubushi, where finite-reservoir or ESN direct
selection offers a clearer advantage. FP also exceeds
ESN Direct$_{500}$ on average at $n=20\,000$ ($0.772$ versus $0.734$),
showing that the analytically tractable linear-recurrence setting remains
competitive in this suite. Full task-level scores are in
Appendix~\ref{sec:full-results}.

\begin{table*}[t]
\centering
\small
\setlength{\tabcolsep}{2.2pt}
\caption{
Average deployment score across ten synthetic benchmarks under the holdout-ridge protocol. Scores are $1-\NRMSE$; higher is better. Mean $\pm$ standard deviation is across tasks. Best and second-best results are bold and underlined, respectively.}
\label{tab:avg-results}
\begin{tabular}{lrrrrrr}
\toprule
Method & Selection rollouts & $n=1\,000$ & $n=3\,000$ & $n=5\,000$ & $n=10\,000$ & $n=20\,000$\\
\midrule
FP & 0 & 0.692$\pm$0.171 & \underline{0.744$\pm$0.159} & \underline{0.756$\pm$0.158} & \underline{0.768$\pm$0.158} & \underline{0.772$\pm$0.158}\\
Memory$_{500}$ & 156{,}600 & 0.619$\pm$0.194 & 0.668$\pm$0.193 & 0.683$\pm$0.197 & 0.698$\pm$0.188 & 0.703$\pm$0.186\\
Direct$_{500}$ & 156{,}600 & \textbf{\boldmath 0.720$\pm$0.178} & \textbf{\boldmath 0.749$\pm$0.164} & \textbf{\boldmath 0.762$\pm$0.158} & \textbf{\boldmath 0.769$\pm$0.157} & \textbf{\boldmath 0.774$\pm$0.156}\\
ESN Direct$_{500}$ & 156{,}600 & \underline{0.708$\pm$0.203} & 0.725$\pm$0.195 & 0.727$\pm$0.188 & 0.731$\pm$0.188 & 0.734$\pm$0.186\\
\bottomrule
\end{tabular}
\end{table*}

\begin{figure}[t]
\centering
\includegraphics[width=\columnwidth]{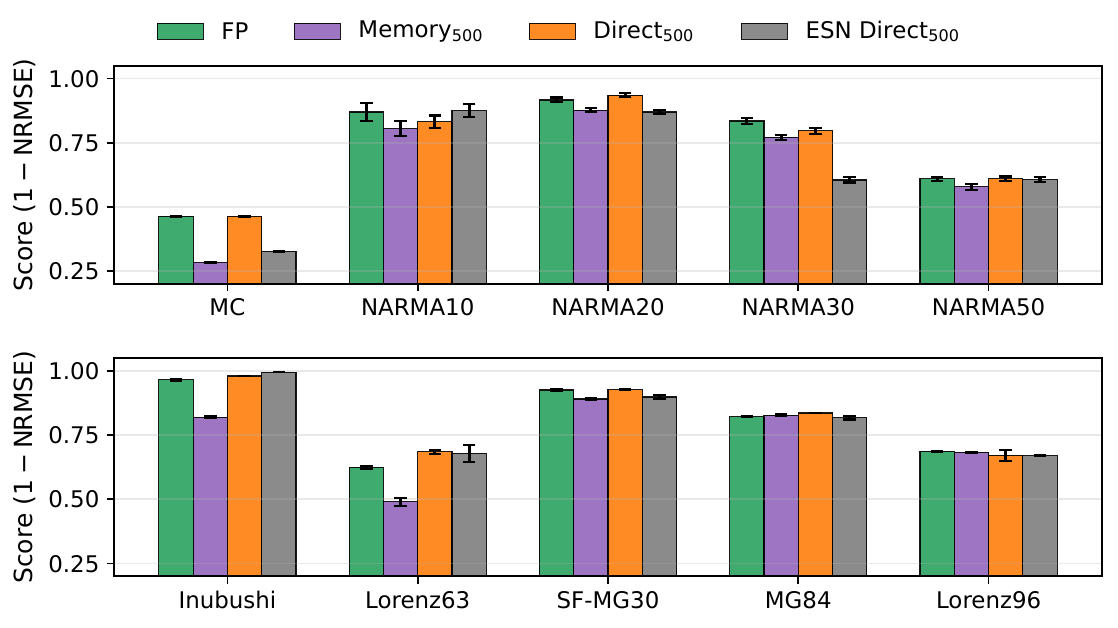}
\caption{Deployment score at $n=20\,000$ across the ten-task synthetic suite under the holdout-ridge protocol. Scores are $1-\NRMSE$; higher is better. Error bars show the standard deviation over deployment seeds.}
\label{fig:bars}
\end{figure}

\paragraph{Average-rank comparison}
Figure~\ref{fig:cd-cost} gives the rank-based
comparison at $n=20\,000$. Direct$_{500}$ has the best mean rank
($1.65$), followed by FP ($2.05$), ESN Direct$_{500}$ ($2.90$), and
Memory$_{500}$ ($3.40$). After Holm correction at $\alpha=0.05$, FP and
Direct$_{500}$ are significantly better than Memory$_{500}$, while the
pairwise differences among FP, Direct$_{500}$, and ESN Direct$_{500}$ are
not significant. The supported conclusion is that FP remains competitive
with the task-direct simulation-based selectors while eliminating
finite-reservoir search rollouts.

\begin{figure}[t]
\centering
\includegraphics[width=\columnwidth]
{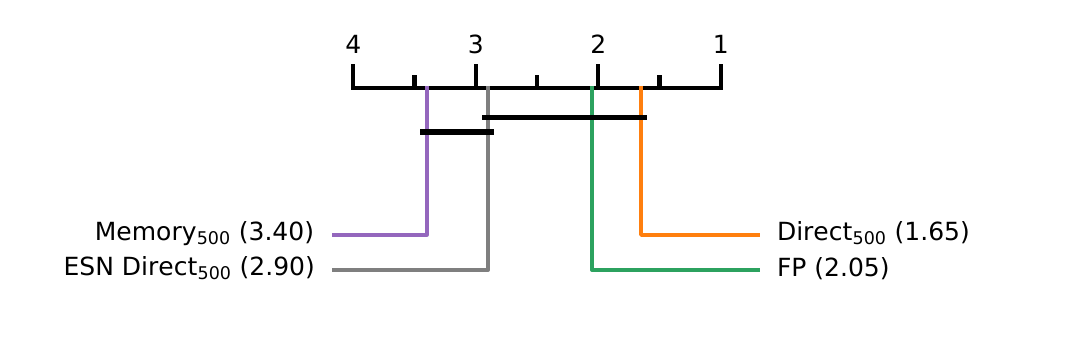}
\caption{Average ranks across the ten synthetic tasks at
$n=20\,000$. Lower is better. Horizontal cliques indicate no
significant pairwise difference under Wilcoxon signed-rank tests with
Holm correction ($\alpha=0.05$).}
\label{fig:cd-cost}
\end{figure}

\paragraph{Width portability and selection cost}
The predictive comparison above is obtained under sharply different
selection costs. Figure~\ref{fig:comp_time_vs_width} reports measured
full-grid selection time at several selection widths. FP does not take
$n_{\rm select}$ as an algorithmic input: its cost depends on the pilot
length, context length, candidate grid, and ridge grid, but not on
reservoir width. We therefore summarize repeated FP timings by a
horizontal mean line; the variation among the individual measurements
reflects run-to-run timing variability rather than width dependence.
The empirical selectors, by contrast, instantiate, roll out, and
evaluate finite reservoirs at every candidate, so their measured cost
increases with $n_{\rm select}$; the connecting lines are only guides
to the eye.

At deployment, the same FP operating point is reused without rerunning
the selector. Its observed mean score increases from $0.692$ at
$n=1000$ to $0.772$ at $n=20\,000$, while the gap to
Direct$_{500}$ narrows from $0.028$ to $0.002$
(Table~\ref{tab:avg-results}). This narrowing is consistent with the
large-width construction: as deployment width increases, the empirical
feature geometry approaches the deterministic kernel targeted by FP.
The result does not imply monotone improvement at every finite width.
The task-level trends in Appendix~\ref{sec:full-results} are consistent with this
interpretation: FP is stable or improving over the reported widths on
every task, whereas selectors tuned at $n_{\rm select}=500$ can exhibit
non-monotone trends on individual tasks.

\begin{figure}[t]
\centering
\includegraphics[width=\columnwidth]
{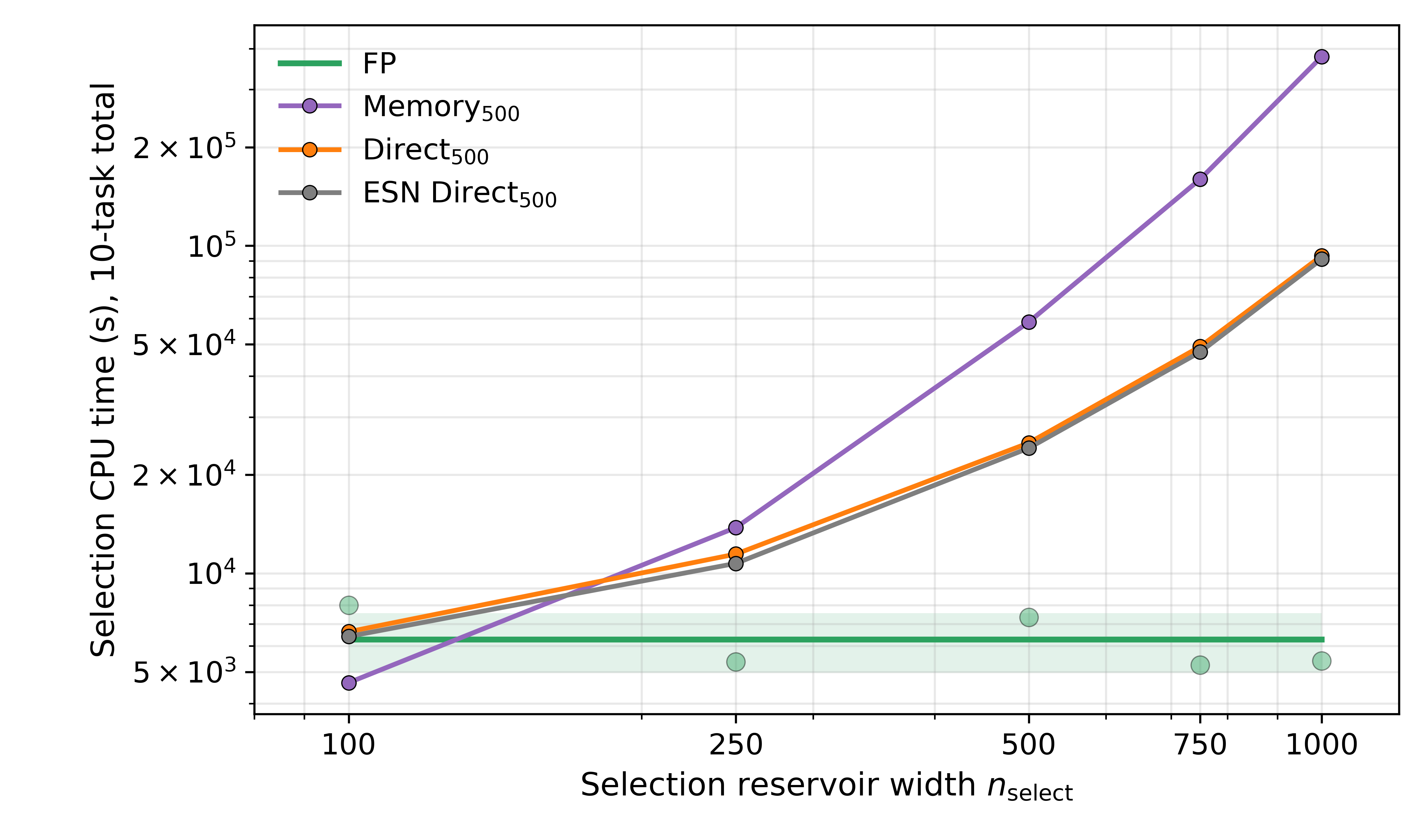}
\caption{Measured aggregate selection CPU time over the ten synthetic
tasks as a function of selection width $n_{\rm select}$. All
finite-reservoir selector points are measured; connecting lines are
guides to the eye and do not represent fitted complexity laws. FP does
not take $n_{\rm select}$ as an input, so its repeated measured runtimes
are summarized by a horizontal mean line, with individual markers and a
shaded $\pm1$ standard-deviation band showing run-to-run timing
variability. FP uses the closed-form erf surrogate employed for
synthetic selection; see Sec.~\ref{sec:fp-kernel}. Lower is better.}
\label{fig:comp_time_vs_width}
\end{figure}

\paragraph{Budgeted screening}
The zero-rollout selector can also serve as a pre-screen when a small
simulation budget is available. FP first ranks the full grid without
finite-reservoir rollouts; Direct$_{500}$ is then evaluated only on the
top $K$ candidates. Table~\ref{tab:fp-screened-synthetic} compares this
FP-K strategy with Random-K and TPE-K, using the same retained candidate
budget but reporting Random/TPE over 20 draws to estimate baseline
variability. FP-K gives the strongest mean performance at every reported budget. The
advantage is largest at small $K$ and shrinks as $K$ grows, indicating
convergence of the budgeted selectors. After Holm correction across the
eight FP-vs-baseline budget comparisons, the paired differences are
directional rather than significant. The main advantage of FP-K is
therefore budget efficiency and determinism: at $K=250$, FP-K matches
the full-grid Direct$_{500}$ reference within sampling variation
($0.777$ versus $0.774$) using $750$ finite-reservoir selection rollouts
per task ($4.8\%$ of the full grid). Random-K and TPE-K use the same
per-run candidate budget, but their reported means are estimated over
20 stochastic draws and reach $0.750$ and $0.769$, respectively. The
screened protocol therefore provides an intermediate operating mode
between zero-rollout selection and exhaustive task-direct search, and
motivates the low-budget Telco experiments (FP-$5$, FP-$10$) below.

\begin{table}[t]
\centering
\small
\setlength{\tabcolsep}{3pt}
\caption{
Budgeted synthetic comparison at $n=20\,000$. Rows indicate the retained candidate budget $K$; scores are deployment $1-\NRMSE$ (higher is better), averaged across the ten tasks. Percentages are relative to exhaustive Direct500 selection.}
\label{tab:fp-screened-synthetic}
\begin{tabular}{rrrrr}
\toprule
$K$ & Ten-task rollouts & FP-K  & Random-K  & TPE-K \\
\midrule
25 & 750 (0.5\%) & \textbf{\boldmath 0.774$\pm$0.159} & 0.669$\pm$0.234 & 0.660$\pm$0.244\\
50 & 1{,}500 (1.0\%) & \textbf{\boldmath 0.779$\pm$0.162} & 0.701$\pm$0.232 & 0.716$\pm$0.214\\
100 & 3{,}000 (1.9\%) & \textbf{\boldmath 0.775$\pm$0.168} & 0.732$\pm$0.219 & 0.742$\pm$0.209\\
250 & 7{,}500 (4.8\%) & \textbf{\boldmath 0.777$\pm$0.166} & 0.750$\pm$0.194 & 0.769$\pm$0.167\\
\midrule
\multicolumn{2}{l}{Full Direct$_{500}$ reference} & \multicolumn{3}{c}{0.774$\pm$0.156}\\
\bottomrule
\end{tabular}
\end{table}

\paragraph{Selection-length sensitivity}
An auxiliary coarse-grid sweep over the selection-sequence length
(Appendix~\ref{sec:timing}) shows that FP is close to its highest peak already at
$T_{\rm select}=250$ and reaches $0.771$ by $T_{\rm select}=500$,
whereas direct selection continues to benefit from longer sequences
as additional
task-specific validation data provide a more reliable ranking of the
candidate reservoirs. Memory$_{500}$ remains constant because its
task-agnostic memory proxy selects the same candidate at every tested
selection length. FP is therefore efficient
not only in finite-reservoir rollouts but also in the length of the
labelled pilot sequence required for selection.

\subsection{Real-data benchmark results}
\label{sec:real-data}

We complement the synthetic suite with real-data forecasting experiments
derived from five datasets: the four public ETT-small datasets
(ETTh1, ETTh2, ETTm1, and ETTm2) and an operational cellular-traffic dataset for ten randomly selected network cells. The corresponding protocol details are given in the real-data benchmark
protocol of Section~\ref{sec:experiments}. The ETT datasets provide reproducible public benchmarks, while the Telco task tests the selector in an operational multivariate setting where finite-reservoir validation rollouts are costly. We therefore focus on two questions: whether zero-rollout FP selection remains competitive with exhaustive task-direct search, and whether FP-based screening allocates small rollout budgets more effectively than Random-K and TPE-K.

\paragraph{Zero-rollout selection}
Table~\ref{tab:forecasting-summary} reports dataset-average scores for
all selectors; horizon-level results are in Appendix~\ref{sec:forecasting-horizons}. On
ETTm1, zero-rollout FP is the strongest of the
linear-recurrence selectors ($0.505$ versus $0.499$ for full-grid
Direct$_{500}$), while ESN Direct$_{500}$ is slightly higher ($0.508$).
At the longest ETTm horizons, the task-agnostic Memory$_{500}$ proxy is
the strongest selector (Appendix~\ref{sec:forecasting-horizons}). This reflects a
short-versus-long-horizon trade-off under aggregate selection. On
ETTm2, for example, Memory$_{500}$ improves the 6- and 12-hour scores
to $0.474$ and $0.501$, compared with $0.456$ and $0.442$ for FP, but
reduces the 1-hour score from $0.862$ to $0.762$. On the ETT dataset
averages, Memory$_{500}$ remains below the task-informed selectors
throughout (Table~\ref{tab:forecasting-summary}). These results also
indicate that the ETT candidate landscape is relatively flat: several
operating points have nearly identical deployment scores, so stochastic
search can catch up as the budget grows.

Telco is different. On the three representative horizons shown in
Table~\ref{tab:forecasting-summary}, zero-rollout FP trails full-grid
Direct$_{500}$ ($0.572$ versus $0.635$), and the task-agnostic
Memory$_{500}$ proxy is also higher ($0.619$). This shows that
zero-rollout ranking alone is not uniformly sufficient on real data.
However, the budgeted aggregate comparison in
Table~\ref{tab:forecasting-budgeted} shows that only a very small
task-direct budget is needed to close the gap: FP-$5$ reaches
$0.622\pm0.112$ against the full-grid Direct$_{500}$ reference
$0.621\pm0.113$, using 15 rather than 462 selection rollouts per cell.

\begin{table}[t]
\centering
\small
\setlength{\tabcolsep}{3pt}
\caption{Three-horizon real-data forecasting summary at $n=20\,000$.
Each entry is the mean$\pm$standard deviation over the three reported
physical horizons (1, 6, and 12 h). Horizon-level scores, whose
deviations are instead over deployment seeds (ETT) or cells (Telco),
are given in Sec.~S-VII of the SM. The Telco aggregate over horizons
$1{:}12$ is reported in Table~\ref{tab:forecasting-budgeted}. Scores are
$1-\NRMSE$; higher is better.}
\label{tab:forecasting-summary}
\begin{tabular}{lcccc}
\toprule
Dataset & FP & Memory$_{500}$ & Direct$_{500}$ & ESN Direct$_{500}$\\
\midrule
ETTh1
& 0.556$\pm$0.16
& 0.522$\pm$0.11
& \textbf{\boldmath 0.567$\pm$0.17}
& \underline{0.564$\pm$0.18}\\
ETTh2
& \textbf{\boldmath 0.664$\pm$0.15}
& 0.525$\pm$0.05
& \textbf{\boldmath 0.664$\pm$0.15}
& \underline{0.662$\pm$0.15}\\
ETTm1
& \underline{0.505$\pm$0.19}
& 0.498$\pm$0.18
& 0.499$\pm$0.19
& \textbf{\boldmath 0.508$\pm$0.20}\\
ETTm2
& \textbf{\boldmath 0.587$\pm$0.24}
& \underline{0.579$\pm$0.16}
& \textbf{\boldmath 0.587$\pm$0.24}
& \textbf{\boldmath 0.587$\pm$0.24}\\
Telco
& 0.572$\pm$0.03
& \underline{0.619$\pm$0.05}
& \textbf{\boldmath 0.635$\pm$0.04}
& 0.607$\pm$0.03\\
\bottomrule
\end{tabular}
\end{table}

\paragraph{Budgeted screening}
As on the synthetic suite, FP-K ranks the grid deterministically at
zero rollout cost and spends the task-direct budget only on the
retained candidates. Table~\ref{tab:forecasting-budgeted} compares FP-K
with Random-K and TPE-K at matched candidate budgets: each individual
selector evaluates $K$ candidates using three reservoir seeds, hence
$3K$ finite-reservoir rollouts. Random-K and TPE-K are repeated for ten
independent draws to estimate their variability; these repetitions use
$30K$ reported rollouts in total but do not increase the candidate
budget available to any individual selector. Their standard deviations
therefore quantify selection instability: a single stochastic run has
the same $3K$ budget but can be substantially below the reported mean.

\begin{table*}[t]
\centering
\small
\setlength{\tabcolsep}{5.0pt}
\caption{
Budgeted real-data forecasting comparison at $n=20\,000$. Each selector evaluates $K$  candidates using $3K$ finite-reservoir rollouts. $\dagger$ denotes FP-K significantly outperforming both Random-K and TPE-K after Holm correction.}
\label{tab:forecasting-budgeted}
\begin{tabular}{llccccc}
\toprule
Dataset & Selector
& \shortstack{$K=5$\\15 rollouts}
& \shortstack{$K=10$\\30 rollouts}
& \shortstack{$K=25$\\75 rollouts}
& \shortstack{$K=50$\\150 rollouts}
& Full Direct$_{500}$\\
\midrule
ETTh1
 & FP-K     & \textbf{0.567$\pm$0.00035} & 0.567$\pm$0.00035
            & 0.567$\pm$0.00035 & 0.567$\pm$0.00035 & 0.567$\pm$0.00035\\
 & Random-K & 0.561$\pm$0.008 & 0.563$\pm$0.009
            & \textbf{0.568$\pm$0.002} & \textbf{0.567$\pm$0.002} & \\
 & TPE-K    & 0.561$\pm$0.013 & \textbf{0.571$\pm$0.007}
            & 0.567$\pm$0.005 & 0.567$\pm$0.005 & \\
\addlinespace[2pt]
ETTh2
 & FP-K     & \textbf{0.665$\pm$0.00049} & \textbf{0.665$\pm$0.00049}
            & \textbf{0.665$\pm$0.00049} & \textbf{0.665$\pm$0.00049}
            & 0.665$\pm$0.00049\\
 & Random-K & 0.609$\pm$0.047 & 0.645$\pm$0.015
            & 0.654$\pm$0.011 & 0.661$\pm$0.003 & \\
 & TPE-K    & 0.626$\pm$0.039 & 0.635$\pm$0.032
            & 0.639$\pm$0.026 & 0.660$\pm$0.007 & \\
\addlinespace[2pt]
ETTm1
 & FP-K     & \textbf{0.501$\pm$0.00079} & \textbf{0.499$\pm$0.00086}
            & 0.499$\pm$0.00086 & 0.499$\pm$0.00086
            & 0.499$\pm$0.00086\\
 & Random-K & 0.493$\pm$0.013 & 0.495$\pm$0.008
            & 0.498$\pm$0.009 & 0.499$\pm$0.007 & \\
 & TPE-K    & 0.493$\pm$0.014 & 0.497$\pm$0.008
            & \textbf{0.503$\pm$0.006} & \textbf{0.501$\pm$0.004} & \\
\addlinespace[2pt]
ETTm2
 & FP-K     & \textbf{0.587$\pm$0.00020} & \textbf{0.587$\pm$0.00020}
            & \textbf{0.587$\pm$0.00020} & \textbf{0.587$\pm$0.00020}
            & 0.587$\pm$0.00020\\
 & Random-K & 0.579$\pm$0.014 & 0.581$\pm$0.012
            & 0.579$\pm$0.004 & 0.584$\pm$0.002 & \\
 & TPE-K    & 0.569$\pm$0.011 & 0.571$\pm$0.008
            & 0.583$\pm$0.005 & 0.584$\pm$0.003 & \\
\addlinespace[2pt]
Telco
 & FP-K     & \textbf{0.622$\pm$0.112}\textsuperscript{$\dagger$}
            & \textbf{0.622$\pm$0.110}\textsuperscript{$\dagger$}
            & \textbf{0.623$\pm$0.110}\textsuperscript{$\dagger$}
            & \textbf{0.621$\pm$0.113} & 0.621$\pm$0.113\\
 & Random-K & 0.548$\pm$0.170 & 0.584$\pm$0.142
            & 0.598$\pm$0.136 & 0.617$\pm$0.116 & \\
 & TPE-K    & 0.572$\pm$0.144 & 0.593$\pm$0.120
            & 0.602$\pm$0.127 & 0.612$\pm$0.114 & \\
\bottomrule
\end{tabular}
\end{table*}

On the public ETT benchmarks, FP-K recovers the full-grid
Direct$_{500}$ operating point already at $K=5$ on ETTh1, ETTh2, and
ETTm2, and selects the same operating point at every reported budget.
TPE-K attains its strongest ETTh1 test score at $K=10$, while the three
screening methods converge at larger budgets. On ETTh2, Random-K and
TPE-K improve with $K$ but remain below FP-K even at $K=50$; on ETTm2
they trail at small budgets and approach FP-K only at larger $K$.

The operational Telco task is more discriminative and provides the
strongest evidence for pilot-informed screening. Full-grid
Direct$_{500}$ reaches $0.621\pm0.113$ across cells, at
$154\times3=462$ rollouts per cell. FP-$5$ already matches or slightly exceeds this
reference ($0.622\pm0.112$) with 15 rollouts, compared with $0.548$
for Random-K and $0.572$ for TPE-K at the same budget, and FP-K
remains between $0.621$ and $0.623$ at every reported budget, compared
with $0.584$ and $0.593$ for the baselines at $K=10$.

FP-K beats both baselines on all ten cells at $K=5$, $K=10$, and
$K=25$; after Holm correction over the eight FP-vs-baseline budget
comparisons, the adjusted values satisfy $p_{\rm Holm}\leq0.0078$ in
all three cases. FP-50 recovers the full-grid Direct$_{500}$ operating
point with 150 rather than 462 rollouts, and by $K=50$ the methods have
largely converged and the corrected differences are no longer
significant.

Taken together, the forecasting results show that FP ranking is most
valuable when validation rollouts are scarce. It either recovers the
full-grid operating point immediately, as on several ETT datasets, or
allocates a small task-direct budget to substantially better candidates
than Random-K and TPE-K, as on the operational Telco task.

\section{Discussion}
\label{sec:discussion}
 
\paragraph{ 
Width-independent tuning of reservoir dynamics}
The large-width kernel relocates selection cost from the reservoir to
the pilot sequence. Empirical selectors instantiate a reservoir for
every grid point; FP instantiates none. This has two concrete
consequences. First, the selection cost is independent of reservoir
width: the same $\theta^\star$ computed at pilot length
$T_{\rm pilot}=500$ can be deployed at any $n$ without rerunning
selection. In our experiments, its mean score improves over the
reported widths, consistently with convergence toward the large-width
kernel. Second, cost scales with pilot length rather than with
reservoir width. The
selection-length study (Appendix~\ref{sec:timing}) shows that FP plateaus at
$T_{\rm select}=500$, while task-direct selection continues to improve;
FP is therefore data-efficient as well as simulation-efficient.

\paragraph{Deterministic screening versus stochastic search}
Random-K and TPE-K are important baselines because they represent the
finite-rollout strategies a practitioner would normally try instead of
full-grid selection. Their behavior also clarifies the role of FP. FP is
not a replacement for high-budget sequential search; rather, it provides
a deterministic ranking of the full candidate grid before any reservoir
is instantiated. The mechanism behind the low-budget advantage is that
the FP ranking is task-informed from the outset: the pilot kernel
already encodes the temporal structure of the downstream task, so the
first $K$ candidates evaluated are concentrated in the relevant region
of the grid. Random-K carries no task information, and TPE-K must spend
its earliest trials learning the response surface before it can exploit
it. The budgeted results show that this ranking is most
valuable at small $K$. On the synthetic suite, FP-K gives the strongest
mean score at every reported budget and matches the full-grid
Direct$_{500}$ reference within sampling variation using $4.8\%$ of the full-grid rollout budget. On Telco, exact-$\tanh$ FP-K is already at the full-grid
Direct$_{500}$ level with $K=5$ and remains stable across the reported
budgets, while Random-K and TPE-K require larger budgets to approach the
same operating point. The low-budget differences against both stochastic
baselines are significant for $K\leq25$ after Holm correction
($p_{\rm Holm}\leq0.0078$).
 On flatter ETT landscapes, the stochastic baselines catch up
as $K$ grows, but their reported means are expectations over stochastic
selector runs; a single run receives the same $3K$ candidate budget and
can vary substantially, as reflected by the standard deviations. Thus,
the practical distinction is not only cost, but also determinism and
stability of the selected operating point.

\paragraph{
Recurrent topology as a kernel design variable}
The theory requires two properties of the recurrent matrix:
asymptotic operator-norm control and diagonal self-averaging of the
entries of $A_n^k(A_n^\ell)^\top$ (Lemma~\ref{lem:diag}). Ginibre
matrices satisfy both with probability tending to one. The same
framework extends to cyclic-shift and Haar-orthogonal recurrences, and
more generally to ensembles with deterministic cross-lag propagation
coefficients; the corresponding statements and coefficient catalogue are
provided in Appendices~\ref{app:haar-diagonal-proof}--\ref{app:recurrent-matrix-kernels}. In this sense, the recurrent
matrix is not merely an implementation detail: once its limiting memory
coefficients are specified, it becomes part of the analytic kernel used
for selection.

\paragraph{Scope and extensions}
The exact theory requires linear recurrent dynamics. For a nonlinear
ESN, the state is no longer a linear combination of propagated inputs,
and the large-width covariance involves time-varying activation factors
that depend on the input trajectory. A frozen-diagonal or mean-field
approximation would replace the exact trace moments with approximate
ones, at the cost of the exactness guarantee. Two further limitations
delimit the intended use of the selector. First, the pilot sequence
must be drawn from the same task distribution as deployment: FP ranks
operating points for the task the pilot describes, and its behavior
under pilot--deployment distribution shift has not been studied here.
Second, FP ranks a fixed discrete candidate grid and does not refine
candidates continuously; a natural hybrid uses FP to screen the grid
and a sequential method to refine locally within the retained region.
Extending the framework to nonlinear recurrences, deep or stacked
reservoir layers, and structured state-space models are natural
directions for future work.

\section{Conclusion}
\label{sec:conclusion}

We introduced a deterministic zero-rollout hyperparameter selector for
reservoir-based temporal sequence learning. The selector uses a
free-probability large-width kernel: the Gram matrix of nonlinear
temporal features converges to a deterministic limit computable from a
short labelled pilot sequence and the cross-lag propagation 
coefficients of
the reservoir ensemble. This moves hyperparameter search offline: no
reservoir is instantiated during selection, the cost is independent of
deployment width, and the selected configuration transfers across widths
without rerunning the search.

Across ten temporal benchmarks covering nonlinear sequence processing,
memory, and chaotic forecasting, Holm-corrected pairwise tests find no
significant difference between FP and the task-direct simulation-based
selectors requiring $156\,600$ finite-reservoir selection rollouts,
while FP is significantly better than the task-agnostic Memory
baseline. In the
budgeted synthetic sweep, FP-K gives the strongest mean score at each
reported budget and matches the full-grid Direct$_{500}$ reference
within sampling variation using only $4.8\%$ of the full-grid rollout budget. On public ETT data,
FP-K recovers the full-grid Direct$_{500}$ operating point on three of
four datasets already at $K=5$. On the ten-cell Telco forecasting task, FP-5 slightly exceeds full-grid
Direct$_{500}$ ($0.622$ versus $0.621$) using 15 rather than 462
selection rollouts per cell, and FP-K is significantly better than
Random-K/TPE-K for $K\leq25$ after Holm correction
($p_{\rm Holm}\leq0.0078$).

The results support free-probability kernels as a practical tool for
pilot-informed reservoir model selection: they are most useful when
finite-reservoir validation rollouts are scarce, while remaining
compatible with higher-budget direct or sequential search. Future work
will extend the construction to nonlinear recurrences and broader
classes of structured sequence models.

\section*{Acknowledgments}
This work was supported by the Norwegian Research Council SFI NorwAI (309834); by EMERGE, a project funded by the European Innovation Council (prj. code 101070918); and by NEURONE, a project funded by the European Union -- Next Generation EU, M4C1 CUP I53D23003600006, under program PRIN 2022 (prj. code 20229JRTZA).

\appendices

\section{Fixed-Context Kernel}
\label{app:fixed-context-proofs}

\subsection{Proof roadmap and notation}

The proof follows four steps. First, the mixed propagation traces
converge to deterministic coefficients $\tau_{k,\ell}$. Second, the
corresponding diagonal entries converge to the same coefficients on
average over reservoir coordinates. Third, conditional on the recurrent
matrix, each coordinate pair is Gaussian; the diagonal result therefore
identifies the limiting nonlinear feature kernel. Finally,
concentration with respect to the Gaussian input matrix removes the
remaining finite-width fluctuations.

We use the same notation as in the main text. For each width $n$, let
\[
W_r^{(n)}
=
\frac{\sigma_r}{\sqrt n}G_r^{(n)},
\qquad
A_n=(1-\alpha)I_n+\alpha W_r^{(n)},
\]
where the entries of $G_r^{(n)}$ are independent standard Gaussian
variables. Independently,
\[
W_{\rm in}^{(n)}
=
\frac{\sigma_{\rm in}}{\sqrt{d_{\rm in}}}G_{\rm in}^{(n)},
\]
with $G_{\rm in}^{(n)}$ again standard Gaussian. For deterministic
$u,v\in\mathbb R^{d_{\rm in}}$,
\begin{equation}
\mathbb E_{W_{\rm in}}\!\left[
W_{\rm in}^{(n)}uv^\top
(W_{\rm in}^{(n)})^\top
\right]
=
\frac{\sigma_{\rm in}^2}{d_{\rm in}}
(u^\top v)I_n.
\label{eq:sm-input-isotropy}
\end{equation}

For a fixed maximum retained lag $L$, define
\begin{equation}
x_t^{(L)}
:=
\alpha\sum_{k=0}^{L}
A_n^kW_{\rm in}^{(n)}u_{t-1-k}.
\label{eq:sm-truncated-state}
\end{equation}
Thus, lags $0,\ldots,L$ are retained, for up to $L+1$ input time
points. For fixed $k,\ell\geq0$, set
\begin{equation}
\tau_{k,\ell}
:=
\sum_{j=0}^{\min\{k,\ell\}}
\binom{k}{j}\binom{\ell}{j}
(1-\alpha)^{k+\ell-2j}
(\alpha\sigma_r)^{2j}.
\label{eq:sm-tau}
\end{equation}
The deterministic linear-state covariance and its associated
$2\times2$ matrix are
\begin{equation}
Q_\theta^{(L)}(t,s)
:=
\alpha^2\frac{\sigma_{\rm in}^2}{d_{\rm in}}
\sum_{k,\ell=0}^{L}
\tau_{k,\ell}\,
u_{t-1-k}^\top u_{s-1-\ell},
\label{eq:sm-Q}
\end{equation}
and
\begin{equation}
\Sigma_\theta^{(L)}(t,s)
:=
\begin{pmatrix}
Q_\theta^{(L)}(t,t) & Q_\theta^{(L)}(t,s)\\
Q_\theta^{(L)}(t,s) & Q_\theta^{(L)}(s,s)
\end{pmatrix}.
\label{eq:sm-Sigma}
\end{equation}
For a coordinate-wise feature map $\psi$, define
\begin{equation}
K_{\psi,n,\theta}^{(L)}(t,s)
:=
\frac1n
\psi(x_t^{(L)})^\top\psi(x_s^{(L)}),
\label{eq:sm-empirical-kernel}
\end{equation}
and
\begin{equation}
K_{\psi,\theta}^{(L)}(t,s)
:=
\mathbb E\!\left[\psi(g_t)\psi(g_s)\right].
\label{eq:sm-deterministic-kernel}
\end{equation}
Here
$(g_t,g_s)\sim
\mathcal N(0,\Sigma_\theta^{(L)}(t,s))$.

Throughout, the input sequence is deterministic and bounded:
$U:=\sup_t\lVert u_t\rVert<\infty$. The symbol $C$ denotes a finite
constant that may change from line to line, may depend on fixed model
parameters, $L$, $U$, and $\psi$, but never on $n$.

We use repeatedly the standard normalized-Ginibre bound
\begin{equation}
\sup_n
\mathbb E\!\left[
\left\|
\frac{G_r^{(n)}}{\sqrt n}
\right\|^m
\right]
<\infty
\qquad
\text{for every fixed }m.
\label{eq:sm-uniform-norm-moments}
\end{equation}
Indeed, the expected norm is at most $2+o(1)$ and Gaussian
concentration gives a sub-Gaussian upper tail
~\cite{vershynin2018high}. Consequently, every fixed power of $A_n$
has uniformly bounded norm moments.

\subsection{Mixed propagation moments}

\begin{proposition}[Ginibre recurrent memory coefficients]
\label{prop:sm-moments}
For every fixed $k,\ell\geq0$,
\begin{equation}
\frac1n\Tr\!\left(
A_n^k(A_n^\ell)^\top
\right)
\xrightarrow{L^1}
\tau_{k,\ell}.
\label{eq:sm-moment-limit}
\end{equation}
\end{proposition}

\begin{proof}
For the proof, write $X_n:=G_r^{(n)}/\sqrt n$. Since $I_n$ commutes
with $X_n$, with the binomial theorem we can write $n^{-1}\Tr\!\left(A_n^k(A_n^\ell)^\top\right)$ as
\begin{equation}
\sum_{i=0}^{k}\sum_{j=0}^{\ell}
\binom{k}{i}\binom{\ell}{j}
(1-\alpha)^{k+\ell-i-j}
(\alpha\sigma_r)^{i+j}
m_n(i,j),
\label{eq:sm-binomial-expansion}
\end{equation}
where
$m_n(i,j):=n^{-1}\Tr(X_n^i(X_n^\top)^j)$.

Normalized real Ginibre matrices converge in expected
$*$-moments to a standard circular element $c$
~\cite[Sec.~11.6.3]{Mingo2017}. The free Wick rule gives
\[
\varphi\!\left(c^i(c^*)^j\right)=\delta_{ij}:
\]
a non-crossing pairing of the ordered word
$c^i(c^*)^j$ exists only when $i=j$, and is then unique. Hence
$\mathbb E\!\left[m_n(i,j)\right]\to\delta_{ij}$.

It remains to control fluctuations. Differentiating
$m_n(i,j)$ with respect to the Gaussian entries and summing the
resulting position terms yields
\[
\lVert\nabla m_n(i,j)\rVert^2
\leq
\frac{(i+j)^2}{n^2}
\max\bigl\{1,\lVert X_n\rVert^{2(i+j-1)}\bigr\}.
\]
The Gaussian Poincar\'e inequality and
\eqref{eq:sm-uniform-norm-moments} therefore give
$\operatorname{Var}(m_n(i,j))=O(n^{-2})$. Thus
$m_n(i,j)\to\delta_{ij}$ in $L^2$. Substitution into the finite sum
\eqref{eq:sm-binomial-expansion} gives
\eqref{eq:sm-moment-limit} with the coefficients
\eqref{eq:sm-tau}.
\end{proof}

\subsection{Diagonal self-averaging}

\begin{lemma}[Averaged diagonal self-averaging]
\label{lem:sm-diag}
For every fixed $L<\infty$,
\begin{equation}
\max_{0\leq k,\ell\leq L}
\frac1n\sum_{i=1}^{n}
\left|
\left[A_n^k(A_n^\ell)^\top\right]_{ii}
-\tau_{k,\ell}
\right|^2
\xrightarrow{\mathbb P}0.
\label{eq:sm-diag-avg}
\end{equation}
\end{lemma}

\begin{proof}
Fix $k,\ell$ and set
$B_n:=A_n^k(A_n^\ell)^\top$. Simultaneous permutation of the rows and
columns of $G_r^{(n)}$ preserves its law. Hence the diagonal entries of
$B_n$ are identically distributed and
\[
\mathbb E\!\left[(B_n)_{ii}\right]
=
\mathbb E\!\left[\frac1n\Tr B_n\right]
\longrightarrow\tau_{k,\ell}
\]
by Proposition~\ref{prop:sm-moments}.

To control one diagonal entry, expand $B_n$ into the finitely many
matrix words in $X_n:=G_r^{(n)}/\sqrt n$ and $X_n^\top$. For a word
$M=F_1\cdots F_r$, differentiation of $(M)_{11}$ at each occurrence
of $X_n$ or $X_n^\top$ gives
\[
\sum_{a,b}
\left|
\frac{\partial(M)_{11}}
{\partial(G_r^{(n)})_{ab}}
\right|^2
\leq
\frac{r^2}{n}
\max\bigl\{1,\lVert X_n\rVert^{2(r-1)}\bigr\}.
\]
After summing over the finitely many words, Gaussian Poincar\'e and
\eqref{eq:sm-uniform-norm-moments} yield
$\operatorname{Var}((B_n)_{11})=O(n^{-1})$. Therefore
\[
\mathbb E\!\left[
\frac1n\sum_{i=1}^n
|(B_n)_{ii}-\tau_{k,\ell}|^2
\right]
=
\mathbb E\!\left[
|(B_n)_{11}-\tau_{k,\ell}|^2
\right]
\longrightarrow0.
\]
Markov's inequality and a finite union bound over
$0\leq k,\ell\leq L$ prove \eqref{eq:sm-diag-avg}.
\end{proof}

Only the diagonal propagation terms are needed to identify the
marginal Gaussian law of each coordinate pair entering the kernel.
Dependence between distinct reservoir coordinates is controlled
separately by the concentration step below.

\subsection{Nonlinear feature-kernel limit}

\begin{theorem}[Deterministic nonlinear reservoir kernel]
\label{thm:sm-kernel}
Let $\psi\in C_b^1(\mathbb R)$. For fixed $L,t,s$ and a deterministic
bounded input sequence,
\begin{equation}
K_{\psi,n,\theta}^{(L)}(t,s)
\xrightarrow{\mathbb P}
K_{\psi,\theta}^{(L)}(t,s)
\qquad (n\to\infty).
\label{eq:sm-kernel-limit}
\end{equation}
The convergence is with respect to the joint randomness of
$W_r^{(n)}$ and $W_{\rm in}^{(n)}$.
\end{theorem}

\begin{proof}
Write
$M_\psi:=\lVert\psi\rVert_\infty$ and
$L_\psi:=\lVert\psi'\rVert_\infty$.
For a positive-semidefinite $2\times2$ matrix $\Sigma$, let
\[
h(\Sigma)
:=
\mathbb E\!\left[\psi(g_1)\psi(g_2)\right],
\qquad
(g_1,g_2)\sim\mathcal N(0,\Sigma).
\]
Coupling two Gaussian vectors through a common standard normal vector
and using
$\lVert\Sigma^{1/2}-(\Sigma')^{1/2}\rVert
\leq\lVert\Sigma-\Sigma'\rVert^{1/2}$
~\cite[Thm.~X.1.1]{bhatia1997matrix} gives
\begin{equation}
|h(\Sigma)-h(\Sigma')|
\leq
2\sqrt2\,M_\psi L_\psi
\lVert\Sigma-\Sigma'\rVert^{1/2}.
\label{eq:sm-gaussian-continuity}
\end{equation}

Conditional on $A_n$, the pair
$(x_{t,i}^{(L)},x_{s,i}^{(L)})$ is centered Gaussian. Its covariance
entries are
\begin{equation}
c_{n,i}^{(L)}(r,q)
=
\alpha^2\frac{\sigma_{\rm in}^2}{d_{\rm in}}
\sum_{k,\ell=0}^{L}
\left[A_n^k(A_n^\ell)^\top\right]_{ii}
u_{r-1-k}^\top u_{q-1-\ell}.
\label{eq:sm-coordinate-covariance}
\end{equation}
Let $C_{n,i}^{(L)}(t,s)$ be the corresponding $2\times2$ covariance
matrix. By Lemma~\ref{lem:sm-diag},
\[
\frac1n\sum_{i=1}^n
\left\|
C_{n,i}^{(L)}(t,s)
-
\Sigma_\theta^{(L)}(t,s)
\right\|_F^2
\xrightarrow{\mathbb P}0.
\]
Together with \eqref{eq:sm-gaussian-continuity}, this implies
\begin{equation}
\mathbb E\!\left[
K_{\psi,n,\theta}^{(L)}(t,s)\mid A_n
\right]
-
K_{\psi,\theta}^{(L)}(t,s)
\xrightarrow{\mathbb P}0.
\label{eq:sm-conditional-mean}
\end{equation}

It remains to control fluctuations around the conditional mean.
Stack the entries of $G_{\rm in}^{(n)}$ into a standard Gaussian
vector $g$ and write
$x_t^{(L)}=M_tg$, $x_s^{(L)}=M_sg$. Since
\[
\lVert M_t\rVert
\leq
\frac{\alpha\sigma_{\rm in}U}{\sqrt{d_{\rm in}}}
\sum_{k=0}^{L}\lVert A_n^k\rVert,
\]
the conditional Gaussian Poincar\'e inequality gives
\begin{equation}
\operatorname{Var}\!\left(
K_{\psi,n,\theta}^{(L)}(t,s)\mid A_n
\right)
\leq
\frac{C}{n}
\left(
\sum_{k=0}^{L}\lVert A_n^k\rVert
\right)^2.
\label{eq:sm-kernel-conditional-variance}
\end{equation}
The right-hand side is $O_{\mathbb P}(n^{-1})$ by
\eqref{eq:sm-uniform-norm-moments}. Hence the conditional fluctuation
vanishes in probability, and \eqref{eq:sm-conditional-mean} proves
\eqref{eq:sm-kernel-limit}.

For reference, the same conditional-mean calculation before applying
$\psi$, together with the Gaussian quadratic-form variance bound,
also gives
\begin{equation}
\frac1n
(x_t^{(L)})^\top x_s^{(L)}
\xrightarrow{\mathbb P}
Q_\theta^{(L)}(t,s).
\label{eq:sm-linear-state-kernel}
\end{equation}
On any fixed finite pilot set, a union bound gives entrywise and hence
Frobenius-norm convergence of the associated kernel blocks.
\end{proof}

\section{Complete History and Selection Consistency}
\label{app:full-history-proof}

The fixed-context result is the object used directly by the selector.
This section shows that, in the stable regime, increasing the maximum
retained lag recovers the complete washed-out history. The argument has
three parts: a deterministic covariance-tail bound, a uniform
finite-width power bound for Ginibre recurrence, and a final
triangle-inequality passage to the nonlinear kernel.

\subsection{Deterministic complete-history kernel}

\begin{proposition}[Complete-history covariance and context limit]
\label{prop:infinite-context}
Assume $\sigma_r<1$ and
$U=\sup_t\lVert u_t\rVert<\infty$. Define
\begin{equation}
Q_\theta^{(\infty)}(t,s)
:=
\alpha^2\frac{\sigma_{\rm in}^2}{d_{\rm in}}
\sum_{k,\ell\geq0}
\tau_{k,\ell}
u_{t-1-k}^\top u_{s-1-\ell}.
\label{eq:sm-Q-infinity}
\end{equation}
The series converges absolutely. With
\[
r_\theta:=1-\alpha+\alpha\sigma_r^2<1,
\]
\begin{equation}
\left|
Q_\theta^{(\infty)}(t,s)
-
Q_\theta^{(L)}(t,s)
\right|
\leq
\frac{2\sigma_{\rm in}^2U^2}
{d_{\rm in}(1-\sigma_r^2)}
r_\theta^{L+1}.
\label{eq:sm-Q-tail}
\end{equation}
If $K_{\psi,\theta}^{(\infty)}$ is obtained from
$Q_\theta^{(\infty)}$ as in \eqref{eq:sm-deterministic-kernel}, then
\begin{equation}
\left|
K_{\psi,\theta}^{(\infty)}(t,s)
-
K_{\psi,\theta}^{(L)}(t,s)
\right|
\leq
C_\theta
r_\theta^{(L+1)/2}.
\label{eq:sm-kernel-geometric-tail}
\end{equation}
In particular,
\begin{equation}
K_{\psi,\theta}^{(L)}(t,s)
\longrightarrow
K_{\psi,\theta}^{(\infty)}(t,s)
\qquad (L\to\infty).
\label{eq:sm-kernel-context-limit}
\end{equation}
\end{proposition}

\begin{proof}
The coefficients $\tau_{k,\ell}$ are nonnegative and have generating
function
\begin{align}
F(x,y)
&:=
\sum_{k,\ell\geq0}
\tau_{k,\ell}x^ky^\ell
\notag\\
&=
\frac{1}{
(1-(1-\alpha)x)(1-(1-\alpha)y)
-\alpha^2\sigma_r^2xy}.
\label{eq:sm-tau-generating-function}
\end{align}
At $x=y=1$,
\[
\sum_{k,\ell\geq0}\tau_{k,\ell}
=
\frac{1}{\alpha^2(1-\sigma_r^2)},
\]
which proves absolute convergence for bounded inputs. Moreover,
\[
F(x,1)=\frac{1}{\alpha(1-r_\theta x)},
\qquad
\sum_{\ell\geq0}\tau_{k,\ell}
=
\frac1\alpha r_\theta^k.
\]
The complement of $\{0,\ldots,L\}^2$ is contained in
$\{k>L\}\cup\{\ell>L\}$. By symmetry,
\[
\sum_{\max\{k,\ell\}>L}\tau_{k,\ell}
\leq
\frac{2}{\alpha}
\sum_{k>L}r_\theta^k
=
\frac{2}{\alpha^2(1-\sigma_r^2)}
r_\theta^{L+1}.
\]
Multiplying by
$\alpha^2\sigma_{\rm in}^2U^2/d_{\rm in}$ gives
\eqref{eq:sm-Q-tail}. The same bound applies to each entry of
$\Sigma_\theta^{(\infty)}-\Sigma_\theta^{(L)}$.
The Gaussian continuity estimate
\eqref{eq:sm-gaussian-continuity} then yields
\eqref{eq:sm-kernel-geometric-tail}.
\end{proof}

Because the admissible candidate grid and pilot index sets are finite,
the convergence in \eqref{eq:sm-kernel-context-limit} is uniform over
all candidates and pilot-kernel entries used by the selector.

\subsection{Finite-width Ginibre tail}

\begin{proposition}[Uniform Ginibre power bound]
\label{prop:ginibre-full-history}
Let
\[
A_n
=
(1-\alpha)I_n
+
\alpha\frac{\sigma_r}{\sqrt n}G_r^{(n)}
\]
and assume $\sigma_r<1$. For any
$1<R<\sigma_r^{-1}$, set
$q_R:=1-\alpha+\alpha\sigma_rR<1$.
Almost surely, there are finite random constants $C_R$ and $N_R$ such
that, for all $n\geq N_R$ and $m\geq0$,
\begin{equation}
\lVert A_n^m\rVert
\leq
C_Rq_R^m.
\label{eq:ginibre-power-bound}
\end{equation}
Consequently, for bounded input histories,
\begin{equation}
x_t^{(\infty)}
:=
\alpha\sum_{k=0}^{\infty}
A_n^kW_{\rm in}^{(n)}u_{t-1-k}
\label{eq:sm-full-state}
\end{equation}
is well defined for all sufficiently large $n$, and
\begin{equation}
\limsup_{n\to\infty}
\frac1{\sqrt n}
\left\|
x_t^{(\infty)}-x_t^{(L)}
\right\|
\leq
\frac{\alpha C_R\sigma_{\rm in}U}
{\sqrt{d_{\rm in}}(1-q_R)}
q_R^{L+1}
\label{eq:ginibre-state-tail}
\end{equation}
almost surely.
\end{proposition}

\begin{proof}
Write $X_n:=G_r^{(n)}/\sqrt n$ for this proof. The normalized real
Ginibre matrix converges strongly to a standard circular element $c$
~\cite{Schultz2005}. A standard consequence of strong convergence is
uniform resolvent control on compact subsets of
$\mathbb C\setminus\operatorname{spec}(c)$: by hermitization, the
smallest singular value of $zI_n-X_n$ is eventually bounded away from
zero uniformly on such compact sets. Since
$\operatorname{spec}(c)$ is the closed unit disk
~\cite{Mingo2017}, for every $R>1$ there are almost surely finite
$M_R,N_R$ such that
\begin{equation}
\sup_{n\geq N_R}\sup_{|z|=R}
\left\|(zI_n-X_n)^{-1}\right\|
\leq M_R.
\label{eq:uniform-contour-resolvent}
\end{equation}

With $f(z)=1-\alpha+\alpha\sigma_r z$, we have $A_n=f(X_n)$.
The holomorphic functional calculus and
\eqref{eq:uniform-contour-resolvent} give
\begin{align*}
\lVert A_n^m\rVert
&\leq
\frac1{2\pi}
\int_{|z|=R}
|f(z)|^m
\left\|(zI_n-X_n)^{-1}\right\|
|dz|\\
&\leq RM_Rq_R^m.
\end{align*}
Thus \eqref{eq:ginibre-power-bound} holds with $C_R=RM_R$.
Summing the geometric tail and using
\[
\frac{\lVert W_{\rm in}^{(n)}\rVert}{\sqrt n}
\longrightarrow
\frac{\sigma_{\rm in}}{\sqrt{d_{\rm in}}}
\qquad\text{almost surely}
\]
proves \eqref{eq:ginibre-state-tail}.
\end{proof}

\subsection{Complete-history kernel}

\begin{theorem}[Complete-history Ginibre kernel]
\label{thm:sm-full-history-kernel}
Under the assumptions of Theorem~\ref{thm:sm-kernel} and
Proposition~\ref{prop:ginibre-full-history}, for every fixed $t,s$,
\begin{equation}
K_{\psi,n,\theta}^{(\infty)}(t,s)
:=
\frac1n
\psi(x_t^{(\infty)})^\top
\psi(x_s^{(\infty)})
\xrightarrow{\mathbb P}
K_{\psi,\theta}^{(\infty)}(t,s).
\label{eq:sm-full-history-kernel-limit}
\end{equation}
\end{theorem}

\begin{proof}
By boundedness and Lipschitz continuity of $\psi$,
\begin{align}
&
\left|
K_{\psi,n,\theta}^{(\infty)}(t,s)
-
K_{\psi,n,\theta}^{(L)}(t,s)
\right|
\notag\\
&\quad\leq
\frac{M_\psi L_\psi}{\sqrt n}
\left(
\lVert x_t^{(\infty)}-x_t^{(L)}\rVert
+
\lVert x_s^{(\infty)}-x_s^{(L)}\rVert
\right).
\label{eq:sm-kernel-state-tail}
\end{align}
Proposition~\ref{prop:ginibre-full-history} bounds this term
geometrically in $L$, uniformly for all sufficiently large $n$. For
fixed $L$, Theorem~\ref{thm:sm-kernel} gives
\[
K_{\psi,n,\theta}^{(L)}(t,s)
\xrightarrow{\mathbb P}
K_{\psi,\theta}^{(L)}(t,s),
\]
while Proposition~\ref{prop:infinite-context} gives
$K_{\psi,\theta}^{(L)}(t,s)\to
K_{\psi,\theta}^{(\infty)}(t,s)$. The result follows from the triangle
inequality by first choosing $L$ large and then letting $n\to\infty$.
\end{proof}

\subsection{Selection consistency}

\begin{corollary}[Consistency of FP selection]
\label{cor:sm-selection-consistency}
Assume that the deterministic kernel and the finite reservoir use the
same feature map. Fix finite candidate and ridge grids, with all ridge
parameters strictly positive. For any fixed maximum retained lag $L$,
if the deterministic validation score has a unique minimizer, then
finite-width reservoir selection returns the same minimizer with
probability tending to one as $n\to\infty$.

In the stable regime, a unique complete-history optimum is preserved
by all sufficiently large finite contexts, and selection based on the
complete finite-reservoir trajectory converges to this optimum as
$n\to\infty$.
\end{corollary}

\begin{proof}
Let $T$ and $V$ be the fixed pilot training and validation indices.
For $\lambda\geq\lambda_{\min}>0$, the prediction map
\[
(K_{TT},K_{VT})
\longmapsto
K_{VT}(K_{TT}+\lambda I)^{-1}Y_T
\]
is continuous and uniformly Lipschitz on the fixed finite blocks. This
follows from
$\lVert(K_{TT}+\lambda I)^{-1}\rVert\leq\lambda_{\min}^{-1}$
and the resolvent identity. The same is therefore true of the
validation NMSE, assuming its denominator is nonzero.

Theorem~\ref{thm:sm-kernel} gives convergence of every fixed-context
kernel entry. A union bound over the finite pilot, candidate, and ridge
grids yields
\[
\max_{\theta,\lambda}
\left|
\NMSE_{{\rm val},n}^{(L)}(\theta,\lambda)
-
\NMSE_{\rm val}^{(L)}(\theta,\lambda)
\right|
\xrightarrow{\mathbb P}0.
\]
A unique minimizer on a finite grid has a positive gap to the
second-best score. Whenever the uniform error is less than half this
gap, the empirical and deterministic minimizers agree. This proves the
fixed-context statement.

Proposition~\ref{prop:infinite-context} gives uniform convergence of
the finite-context deterministic scores to their complete-history
counterparts over the same finite grids. Hence a unique
complete-history minimizer is also the finite-context minimizer for all
sufficiently large $L$. Finally,
Theorem~\ref{thm:sm-full-history-kernel} gives uniform convergence of
the empirical complete-history scores, and the same positive-gap
argument proves the final claim.
\end{proof}

Near ties have a small score gap and may therefore require larger
reservoir width or maximum lag before the ranking stabilizes.

\section{Haar-Orthogonal Recurrent Matrix}
\label{app:haar-diagonal-proof}

The Ginibre proof uses Gaussian concentration to establish diagonal
self-averaging. For a Haar-orthogonal recurrent matrix, the same
conclusion follows directly from orthogonal invariance.

\begin{proposition}[Haar-orthogonal extension]
\label{prop:haar-extension}
Let $O_n$ be Haar distributed on the orthogonal group, independently of
$W_{\rm in}^{(n)}$, and let
\[
A_n=(1-\alpha)I_n+\alpha\sigma_rO_n.
\]
For every fixed $L<\infty$, Lemma~\ref{lem:sm-diag} holds with the
coefficients $\tau_{k,\ell}$ in \eqref{eq:sm-tau}. Consequently,
Theorem~\ref{thm:sm-kernel} holds with the same deterministic kernel.
If $\sigma_r<1$, the complete-history conclusion also holds.
\end{proposition}

\begin{proof}
Fix $k,\ell$ and set
$B_n=A_n^k(A_n^\ell)^\top$. Orthogonal conjugation preserves the law
of $B_n$. If $v$ is independent and uniform on the unit sphere, then
\[
(B_n)_{11}
\overset d=
v^\top B_nv.
\]
The spherical fourth-moment identity gives
\[
\mathbb E\!\left[
\left|
v^\top B_nv-\frac1n\Tr B_n
\right|^2
\,\middle|\,B_n
\right]
\leq
\frac{2\lVert B_n\rVert^2}{n+2}.
\]
Moreover,
$\lVert B_n\rVert
\leq(1-\alpha+\alpha\sigma_r)^{k+\ell}$.
The normalized traces of fixed Laurent polynomials in $O_n$ converge
to the corresponding Haar-unitary moments
~\cite{CollinsMale2014}, which give the same
$\tau_{k,\ell}$ as in \eqref{eq:sm-tau}. Thus
\[
\mathbb E\!\left[
\left|
(B_n)_{11}-\tau_{k,\ell}
\right|^2
\right]
\longrightarrow0.
\]
The diagonal entries are exchangeable, so averaging over them, followed
by Markov's inequality and a finite union bound, proves the diagonal
self-averaging statement. The fixed-context kernel follows from the
same conditional-Gaussian argument as
Theorem~\ref{thm:sm-kernel}. For $\sigma_r<1$,
\[
\lVert A_n^m\rVert
\leq
(1-\alpha+\alpha\sigma_r)^m,
\]
which gives the complete-history extension directly.
\end{proof}

\section{Other Recurrent Matrix Structures}
\label{app:recurrent-matrix-kernels}

The kernel derivation depends on the recurrent matrix only through the
limits of
\[
\frac1n\Tr\!\left(A_n^k(A_n^\ell)^\top\right)
\]
and the corresponding diagonal self-averaging property. This section
records these coefficients for several recurrent matrix structures.
Once they are known, the covariance and nonlinear kernel are obtained
from \eqref{eq:sm-Q}--\eqref{eq:sm-deterministic-kernel} by replacing
$\tau_{k,\ell}$ with the appropriate coefficient
$\eta_{k,\ell}$. No additional kernel notation is needed.

\begin{remark}[Optional closed-form surrogate]
\label{rem:erf-surrogate}
For any of the real-valued recurrent matrices below, the choice
\[
\psi_{\rm erf}(x)
=
\operatorname{erf}(\sqrt{\pi}x/2)
\]
yields a closed-form kernel. With
\[
D_\theta^{(L)}(t)
:=
2+\pi Q_\theta^{(L)}(t,t),
\]
the kernel is
\begin{equation}
K_{\mathrm{erf},\theta}^{(L)}(t,s)
=
\frac{2}{\pi}\arcsin\!\left(
\frac{\pi Q_\theta^{(L)}(t,s)}
{\sqrt{D_\theta^{(L)}(t)D_\theta^{(L)}(s)}}
\right).
\label{eq:erf-arcsin-sm}
\end{equation}
Here $Q_\theta^{(L)}$ is formed with the relevant coefficients
$\eta_{k,\ell}$. Since
\[
\delta_{\rm erf}
:=
\sup_x
\left|
\tanh(x)-\operatorname{erf}(\sqrt{\pi}x/2)
\right|
<0.0354,
\]
the corresponding limiting kernel entries differ by less than
$2\delta_{\rm erf}<0.071$. The surrogate is used only for the
synthetic selection and timing experiments; all ETT and Telco results
use the exact $\tanh$ kernel evaluated by Gauss--Hermite quadrature.
The empirical validation is reported in
Sec.~\ref{app:surrogate-validation}.
\end{remark}

\subsection{Ginibre, Haar-orthogonal, and cyclic-shift matrices}

For the real Ginibre matrix used in the experiments,
\[
W_r^{(n)}
=
\frac{\sigma_r}{\sqrt n}G_r^{(n)},
\]
Proposition~\ref{prop:sm-moments} gives
$\eta_{k,\ell}=\tau_{k,\ell}$. Proposition~\ref{prop:haar-extension}
shows that the same coefficients hold for
$W_r^{(n)}=\sigma_rO_n$ with $O_n$ Haar orthogonal.

They also hold for the deterministic cyclic-shift matrix $C_n$,
defined by
\[
C_ne_j=e_{j+1\;({\rm mod}\;n)},
\qquad
W_r^{(n)}=\sigma_rC_n.
\]
Indeed, $C_n^\top=C_n^{-1}$ and, for fixed $i,j$ and all sufficiently
large $n$,
\[
\frac1n\Tr\!\left(C_n^i(C_n^\top)^j\right)
=
\delta_{ij}.
\]
The three structures therefore produce the same fixed-context
coefficients. Their complete-history control differs: Ginibre uses
Proposition~\ref{prop:ginibre-full-history}, whereas the orthogonal
matrices satisfy the direct geometric bound
\[
\lVert A_n^m\rVert
\leq
(1-\alpha+\alpha\sigma_r)^m.
\]

\subsection{Circulant matrices}

Let
\[
C_n
=
F_n^*
\operatorname{diag}
(\gamma_{n,0},\ldots,\gamma_{n,n-1})
F_n,
\qquad
W_r^{(n)}=\sigma_rC_n,
\]
where $F_n$ is the unitary discrete Fourier matrix. Assume that the
empirical distribution of the uniformly bounded Fourier eigenvalues
converges weakly to $\nu$. Then
\begin{equation}
\eta_{k,\ell}^{\rm circ}
=
\int_{\mathbb C}
(1-\alpha+\alpha\sigma_r z)^k
(1-\alpha+\alpha\sigma_r\overline z)^\ell
\,d\nu(z).
\label{eq:catalogue-circulant-coefficients}
\end{equation}
If the matrices are generated by a limiting Fourier symbol
$g:[0,2\pi]\to\mathbb C$, define
\[
a_\theta(\omega)
=
1-\alpha+\alpha\sigma_rg(\omega).
\]
Then
\[
\eta_{k,\ell}^{\rm circ}
=
\frac1{2\pi}
\int_0^{2\pi}
a_\theta(\omega)^k
\overline{a_\theta(\omega)}^{\,\ell}
\,d\omega.
\]
Because $A_n^k(A_n^\ell)^\top$ is circulant, every diagonal entry
equals its normalized trace; diagonal self-averaging is therefore
automatic. Complete history follows whenever
\[
\sup_{\omega\in[0,2\pi]}
|a_\theta(\omega)|<1.
\]
The cyclic shift is the special case $g(\omega)=e^{i\omega}$.

\subsection{Gaussian skew-symmetric matrices}

Let
\[
S_n
=
\frac{G_n-G_n^\top}{\sqrt{2n}},
\qquad
W_r^{(n)}=\sigma_rS_n.
\]
Since $S_n^\top=-S_n$, the Hermitian matrix $iS_n$ has the standard
semicircular limit~\cite{Mingo2017}. Writing
\[
d\mu_{\rm sc}(x)
=
\frac{\sqrt{4-x^2}}{2\pi}
\mathbf 1_{[-2,2]}(x)\,dx,
\]
the coefficients are
\begin{equation}
\eta_{k,\ell}^{\rm skew}
=
\int_{-2}^{2}
(1-\alpha+i\alpha\sigma_rx)^k
(1-\alpha-i\alpha\sigma_rx)^\ell
\,d\mu_{\rm sc}(x).
\label{eq:catalogue-skew-coefficients}
\end{equation}
Permutation symmetry and the Gaussian-Poincar\'e argument of
Lemma~\ref{lem:sm-diag} give diagonal self-averaging. Since $A_n$ is
normal, complete history follows under the asymptotic condition
\[
\sqrt{(1-\alpha)^2+4\alpha^2\sigma_r^2}<1.
\]

\subsection{Complex-valued matrices}

For complex recurrent matrices, transpose is replaced by the Hermitian
adjoint. Normalized complex Ginibre and scaled Haar-unitary matrices
again satisfy
\[
\eta_{k,\ell}^{\mathbb C}
:=
\lim_{n\to\infty}
\frac1n\Tr\!\left(A_n^k(A_n^\ell)^*\right)
=
\tau_{k,\ell}
\]
~\cite{Mingo2017,CollinsMale2014}. The corresponding Hermitian
linear-state covariance is obtained from \eqref{eq:sm-Q} by replacing
the input inner product with
$u_{t-1-k}^*u_{s-1-\ell}$. A nonlinear feature kernel additionally
requires a specified complex activation or a realification convention.
The experiments in this paper use real-valued recurrent matrices.

\begin{table*}[t]
\centering
\scriptsize
\setlength{\tabcolsep}{4pt}
\caption{Recurrent-matrix catalogue. The coefficient column gives the
quantities substituted for $\tau_{k,\ell}$ in
\eqref{eq:sm-Q}; the nonlinear kernel is then formed as in
\eqref{eq:sm-deterministic-kernel}.}
\label{tab:recurrent-matrix-catalogue}
\begin{tabularx}{\textwidth}{l l l X}
\toprule
Recurrent matrix & Coefficients & Complete-history control
& Nonlinear-kernel status \\
\midrule
Real Ginibre
& $\eta_{k,\ell}=\tau_{k,\ell}$
& Resolvent bound
& Exact $\tanh$ by quadrature; optional erf surrogate for the
synthetic and timing experiments. \\
Haar orthogonal
& $\eta_{k,\ell}=\tau_{k,\ell}$
& Direct normal-matrix bound
& Same real-valued nonlinear-kernel construction. \\
Cyclic shift
& $\eta_{k,\ell}=\tau_{k,\ell}$ for fixed lags and large $n$
& Direct normal-matrix bound
& Same real-valued nonlinear-kernel construction. \\
Circulant
& Fourier integral \eqref{eq:catalogue-circulant-coefficients}
& Uniform symbol bound
& Same construction with the Fourier coefficients. \\
Gaussian skew-symmetric
& Semicircle integral \eqref{eq:catalogue-skew-coefficients}
& Normal-matrix bound
& Same construction with the skew-symmetric coefficients. \\
Complex Ginibre / Haar unitary
& $\eta_{k,\ell}^{\mathbb C}=\tau_{k,\ell}$
& Resolvent / normal-matrix bound
& Linear covariance explicit; nonlinear kernel requires a complex
feature convention. \\
\bottomrule
\end{tabularx}
\end{table*}

\section{Full Task-Level Deployment Scores}
\label{sec:full-results}

Table~\ref{tab:full-deployment} gives per-task scores for the four
primary selectors, including ESN Direct$_{500}$, at all five
deployment widths ($n\in\{1\,000,3\,000,5\,000,10\,000,20\,000\}$),
while Figure~\ref{fig:score-vs-n} shows the corresponding width
trends.

\begin{table*}[t]
\centering
\small
\setlength{\tabcolsep}{3pt}
\renewcommand{\arraystretch}{0.92}
\caption{Full task-level deployment scores under the holdout-ridge
protocol. Values are $1-\NRMSE$ mean$\pm$std over deployment seeds.
Intermediate widths use three deployment seeds; $n=20\,000$ uses ten
deployment seeds. Best score per task and width is bolded; second-best
is underlined. Identical FP and Direct$_{500}$ rows (e.g., MC) indicate
that both selectors chose the same configuration.}
\label{tab:full-deployment}
\begin{tabular}{llccccc}
\toprule
Task & Selector & $n=1\,000$ & $n=3\,000$ & $n=5\,000$ & $n=10\,000$ & $n=20\,000$\\
\midrule
MC & FP & \textbf{\boldmath 0.322$\pm$0.023} & \textbf{\boldmath 0.418$\pm$0.004} & \textbf{\boldmath 0.439$\pm$0.001} & \textbf{\boldmath 0.457$\pm$0.006} & \textbf{\boldmath 0.463$\pm$0.003}\\
 & Memory$_{500}$ & 0.184$\pm$0.004 & 0.211$\pm$0.004 & 0.214$\pm$0.003 & 0.279$\pm$0.005 & 0.283$\pm$0.001\\
 & Direct$_{500}$ & \textbf{\boldmath 0.322$\pm$0.023} & \textbf{\boldmath 0.418$\pm$0.004} & \textbf{\boldmath 0.439$\pm$0.001} & \textbf{\boldmath 0.457$\pm$0.006} & \textbf{\boldmath 0.463$\pm$0.003}\\
 & ESN Direct$_{500}$ & \underline{0.223$\pm$0.004} & \underline{0.268$\pm$0.015} & \underline{0.302$\pm$0.003} & \underline{0.318$\pm$0.006} & \underline{0.328$\pm$0.002}\\
\addlinespace[1.5pt]
NARMA10 & FP & 0.763$\pm$0.014 & \textbf{\boldmath 0.850$\pm$0.012} & \underline{0.854$\pm$0.004} & \underline{0.871$\pm$0.010} & \underline{0.870$\pm$0.034}\\
 & Memory$_{500}$ & 0.720$\pm$0.014 & 0.782$\pm$0.008 & 0.789$\pm$0.002 & 0.804$\pm$0.004 & 0.805$\pm$0.030\\
 & Direct$_{500}$ & \underline{0.810$\pm$0.012} & 0.820$\pm$0.009 & 0.825$\pm$0.009 & 0.824$\pm$0.009 & 0.832$\pm$0.025\\
 & ESN Direct$_{500}$ & \textbf{\boldmath 0.820$\pm$0.011} & \underline{0.836$\pm$0.004} & \textbf{\boldmath 0.865$\pm$0.009} & \textbf{\boldmath 0.872$\pm$0.007} & \textbf{\boldmath 0.875$\pm$0.025}\\
\addlinespace[1.5pt]
NARMA20 & FP & 0.738$\pm$0.007 & \underline{0.856$\pm$0.010} & \underline{0.880$\pm$0.006} & \underline{0.909$\pm$0.009} & \underline{0.917$\pm$0.009}\\
 & Memory$_{500}$ & 0.725$\pm$0.003 & 0.824$\pm$0.012 & 0.843$\pm$0.006 & 0.871$\pm$0.011 & 0.877$\pm$0.007\\
 & Direct$_{500}$ & \textbf{\boldmath 0.805$\pm$0.005} & \textbf{\boldmath 0.904$\pm$0.005} & \textbf{\boldmath 0.915$\pm$0.004} & \textbf{\boldmath 0.931$\pm$0.007} & \textbf{\boldmath 0.936$\pm$0.007}\\
 & ESN Direct$_{500}$ & \underline{0.739$\pm$0.006} & 0.831$\pm$0.008 & 0.847$\pm$0.006 & 0.862$\pm$0.002 & 0.870$\pm$0.009\\
\addlinespace[1.5pt]
NARMA30 & FP & \textbf{\boldmath 0.618$\pm$0.010} & \textbf{\boldmath 0.730$\pm$0.019} & \textbf{\boldmath 0.773$\pm$0.008} & \textbf{\boldmath 0.824$\pm$0.002} & \textbf{\boldmath 0.835$\pm$0.013}\\
 & Memory$_{500}$ & 0.602$\pm$0.022 & \underline{0.689$\pm$0.007} & 0.734$\pm$0.006 & 0.770$\pm$0.006 & 0.772$\pm$0.010\\
 & Direct$_{500}$ & \underline{0.613$\pm$0.010} & 0.653$\pm$0.056 & \underline{0.745$\pm$0.010} & \underline{0.790$\pm$0.005} & \underline{0.797$\pm$0.011}\\
 & ESN Direct$_{500}$ & 0.607$\pm$0.009 & 0.608$\pm$0.007 & 0.610$\pm$0.008 & 0.609$\pm$0.010 & 0.605$\pm$0.010\\
\addlinespace[1.5pt]
NARMA50 & FP & \underline{0.598$\pm$0.005} & \underline{0.605$\pm$0.004} & \underline{0.606$\pm$0.004} & \underline{0.606$\pm$0.005} & \textbf{\boldmath 0.610$\pm$0.008}\\
 & Memory$_{500}$ & 0.412$\pm$0.045 & 0.538$\pm$0.006 & 0.553$\pm$0.002 & 0.574$\pm$0.004 & 0.578$\pm$0.012\\
 & Direct$_{500}$ & \textbf{\boldmath 0.602$\pm$0.005} & \textbf{\boldmath 0.606$\pm$0.004} & \textbf{\boldmath 0.607$\pm$0.004} & \textbf{\boldmath 0.607$\pm$0.005} & \textbf{\boldmath 0.610$\pm$0.008}\\
 & ESN Direct$_{500}$ & 0.573$\pm$0.008 & 0.590$\pm$0.004 & 0.596$\pm$0.003 & 0.600$\pm$0.004 & \underline{0.607$\pm$0.009}\\
\addlinespace[1.5pt]
Inubushi & FP & 0.948$\pm$0.004 & 0.960$\pm$0.003 & 0.962$\pm$0.002 & 0.962$\pm$0.002 & 0.966$\pm$0.003\\
 & Memory$_{500}$ & 0.793$\pm$0.005 & 0.798$\pm$0.010 & 0.803$\pm$0.007 & 0.811$\pm$0.009 & 0.819$\pm$0.004\\
 & Direct$_{500}$ & \underline{0.969$\pm$0.002} & \underline{0.978$\pm$0.001} & \underline{0.979$\pm$0.001} & \underline{0.979$\pm$0.002} & \underline{0.980$\pm$0.001}\\
 & ESN Direct$_{500}$ & \textbf{\boldmath 0.994$\pm$0.001} & \textbf{\boldmath 0.995$\pm$0.001} & \textbf{\boldmath 0.996$\pm$0.000} & \textbf{\boldmath 0.995$\pm$0.000} & \textbf{\boldmath 0.994$\pm$0.000}\\
\addlinespace[1.5pt]
Lorenz63 & FP & 0.600$\pm$0.006 & 0.620$\pm$0.005 & 0.618$\pm$0.003 & 0.621$\pm$0.002 & 0.624$\pm$0.004\\
 & Memory$_{500}$ & 0.490$\pm$0.013 & 0.486$\pm$0.011 & 0.496$\pm$0.019 & 0.475$\pm$0.003 & 0.490$\pm$0.015\\
 & Direct$_{500}$ & \underline{0.652$\pm$0.007} & \underline{0.665$\pm$0.004} & \underline{0.669$\pm$0.008} & \underline{0.672$\pm$0.001} & \textbf{\boldmath 0.685$\pm$0.007}\\
 & ESN Direct$_{500}$ & \textbf{\boldmath 0.738$\pm$0.003} & \textbf{\boldmath 0.733$\pm$0.023} & \textbf{\boldmath 0.687$\pm$0.012} & \textbf{\boldmath 0.674$\pm$0.044} & \underline{0.679$\pm$0.033}\\
\addlinespace[1.5pt]
SF-MG30 & FP & \underline{0.898$\pm$0.014} & \underline{0.918$\pm$0.003} & \textbf{\boldmath 0.931$\pm$0.003} & \textbf{\boldmath 0.928$\pm$0.002} & \underline{0.926$\pm$0.002}\\
 & Memory$_{500}$ & 0.841$\pm$0.009 & 0.870$\pm$0.020 & 0.887$\pm$0.004 & 0.891$\pm$0.003 & 0.891$\pm$0.004\\
 & Direct$_{500}$ & \textbf{\boldmath 0.904$\pm$0.013} & \textbf{\boldmath 0.925$\pm$0.001} & \underline{0.927$\pm$0.004} & \underline{0.926$\pm$0.001} & \textbf{\boldmath 0.927$\pm$0.002}\\
 & ESN Direct$_{500}$ & 0.893$\pm$0.001 & 0.902$\pm$0.003 & 0.884$\pm$0.009 & 0.898$\pm$0.003 & 0.898$\pm$0.008\\
\addlinespace[1.5pt]
MG84 & FP & 0.812$\pm$0.005 & 0.816$\pm$0.001 & 0.816$\pm$0.002 & 0.819$\pm$0.001 & 0.822$\pm$0.002\\
 & Memory$_{500}$ & 0.786$\pm$0.014 & 0.802$\pm$0.018 & \underline{0.827$\pm$0.003} & \underline{0.823$\pm$0.011} & \underline{0.828$\pm$0.003}\\
 & Direct$_{500}$ & \textbf{\boldmath 0.844$\pm$0.018} & \textbf{\boldmath 0.829$\pm$0.004} & \textbf{\boldmath 0.833$\pm$0.001} & \textbf{\boldmath 0.835$\pm$0.000} & \textbf{\boldmath 0.836$\pm$0.001}\\
 & ESN Direct$_{500}$ & \underline{0.827$\pm$0.006} & \underline{0.824$\pm$0.005} & 0.815$\pm$0.003 & 0.816$\pm$0.001 & 0.818$\pm$0.007\\
\addlinespace[1.5pt]
Lorenz96 & FP & 0.629$\pm$0.004 & 0.666$\pm$0.018 & \textbf{\boldmath 0.683$\pm$0.001} & \textbf{\boldmath 0.686$\pm$0.001} & \textbf{\boldmath 0.686$\pm$0.002}\\
 & Memory$_{500}$ & 0.637$\pm$0.017 & \underline{0.677$\pm$0.002} & \underline{0.680$\pm$0.001} & \underline{0.683$\pm$0.001} & \underline{0.683$\pm$0.002}\\
 & Direct$_{500}$ & \textbf{\boldmath 0.679$\pm$0.003} & \textbf{\boldmath 0.694$\pm$0.001} & \underline{0.680$\pm$0.021} & 0.666$\pm$0.022 & 0.671$\pm$0.023\\
 & ESN Direct$_{500}$ & \underline{0.664$\pm$0.002} & 0.668$\pm$0.000 & 0.669$\pm$0.000 & 0.670$\pm$0.000 & 0.670$\pm$0.000\\
\bottomrule
\end{tabular}
\end{table*}

\begin{figure*}[t]
\centering
\includegraphics[width=\textwidth]{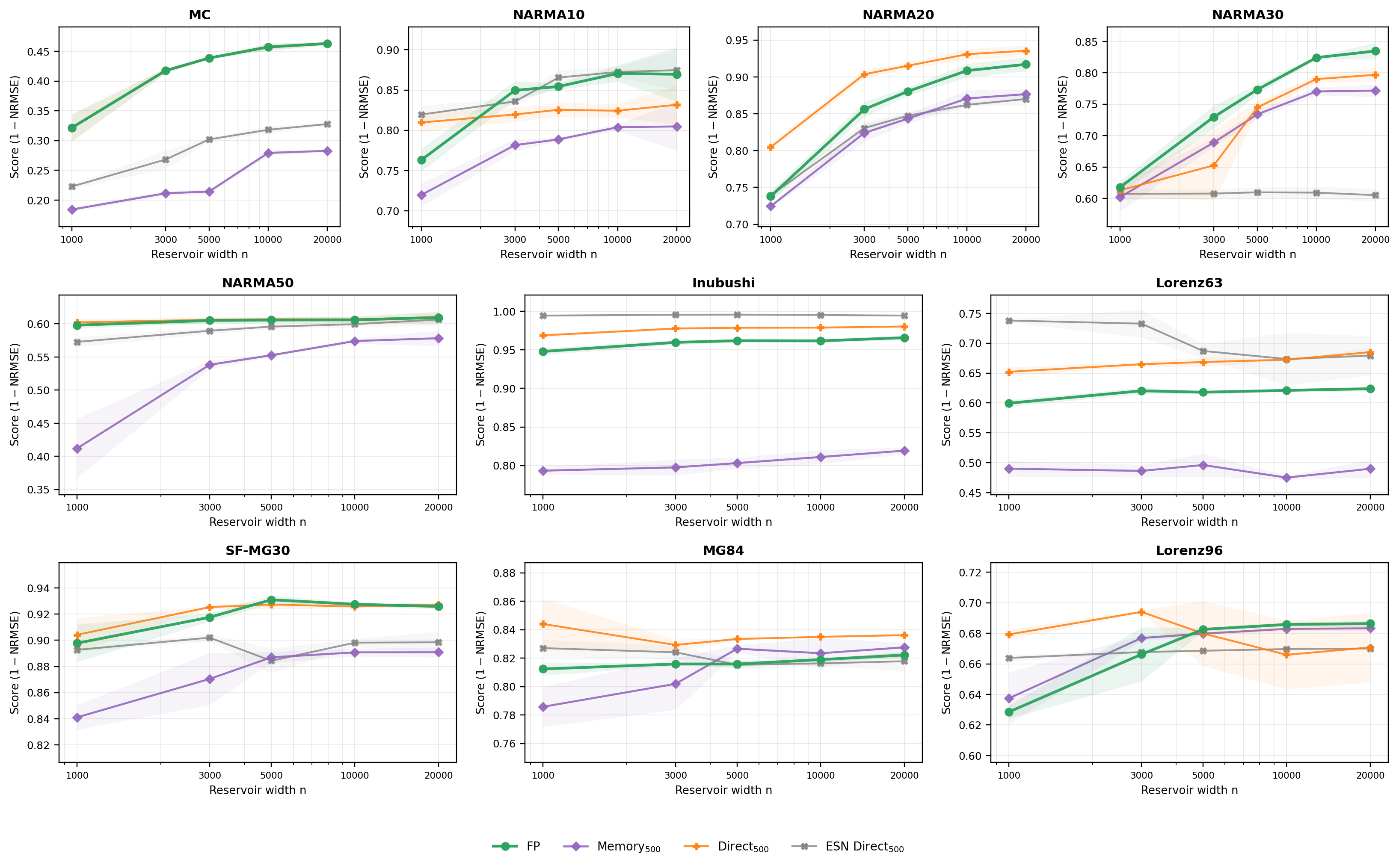}
\caption{Deployment score vs.\ reservoir width $n$ through $n=20\,000$
for the ten-task, four-selector comparison. Shaded bands show $\pm1$
s.d.\ over deployment seeds; the $n=20\,000$ column uses ten deployment
seeds and the intermediate widths use three. Higher is better.}
\label{fig:score-vs-n}
\end{figure*}

\section{Selection-Length Sensitivity and Selection Time}
\label{sec:timing}

Table~\ref{tab:selection-length} and
Figures~\ref{fig:timing-perf}--\ref{fig:timing-avg} report the
auxiliary coarse-grid sweep over the selection-sequence length
$T_{\rm select}$ summarized in Section~\ref{sec:synthetic-results}: deployment
performance at $n=20\,000$ and per-task selection CPU time as a
function of $T_{\rm select}$, at $n_{\rm select}=500$.

\begin{table*}[t]
\centering
\small
\setlength{\textwidth}{4pt}
\caption{Selection-length sensitivity in an auxiliary coarse-grid
synthetic sweep. Scores are deployment $1-\NRMSE$ at $n=20\,000$,
averaged over the ten synthetic tasks; $\pm$ denotes task-level
standard deviation. The selection length $T_{\rm select}$ is used only
during hyperparameter selection. FP uses the closed-form erf-surrogate selection kernel from
Remark~\ref{rem:erf-surrogate}; matched exact-$\tanh$ checks are
reported in Sec.~\ref{app:surrogate-validation}. Memory$_{500}$ uses
$D=\min(T_{\rm select},1000)$ delayed targets and selects the same
candidate throughout this sweep. Best entries are bolded and
second-best entries are underlined.}
\label{tab:selection-length}
\begin{tabular}{@{}lcccc@{}}
\toprule
$T_{\rm select}$ &
FP $\uparrow$ &
Memory$_{500}$ $\uparrow$ &
Direct$_{500}$ $\uparrow$ &
ESN Direct$_{500}$ $\uparrow$\\
\midrule
150
& $0.627{\pm}0.196$
& \textbf{\boldmath $0.727{\pm}0.165$}
& \underline{$0.662{\pm}0.176$}
& $0.640{\pm}0.261$\\
250
& \textbf{\boldmath $0.744{\pm}0.148$}
& \underline{$0.727{\pm}0.165$}
& $0.688{\pm}0.188$
& $0.626{\pm}0.273$\\
500
& \textbf{\boldmath $0.771{\pm}0.158$}
& \underline{$0.727{\pm}0.165$}
& $0.705{\pm}0.147$
& $0.698{\pm}0.191$\\
1000
& \textbf{\boldmath $0.754{\pm}0.185$}
& $0.727{\pm}0.165$
& \underline{$0.746{\pm}0.176$}
& $0.710{\pm}0.190$\\
2000
& \underline{$0.749{\pm}0.202$}
& $0.727{\pm}0.165$
& \textbf{\boldmath $0.761{\pm}0.169$}
& $0.727{\pm}0.192$\\
\bottomrule
\end{tabular}
\end{table*}

\begin{figure*}[t]
\centering
\begin{minipage}[t]{0.48\textwidth}
  \vspace{0pt}
  \centering
  \includegraphics[width=\linewidth]{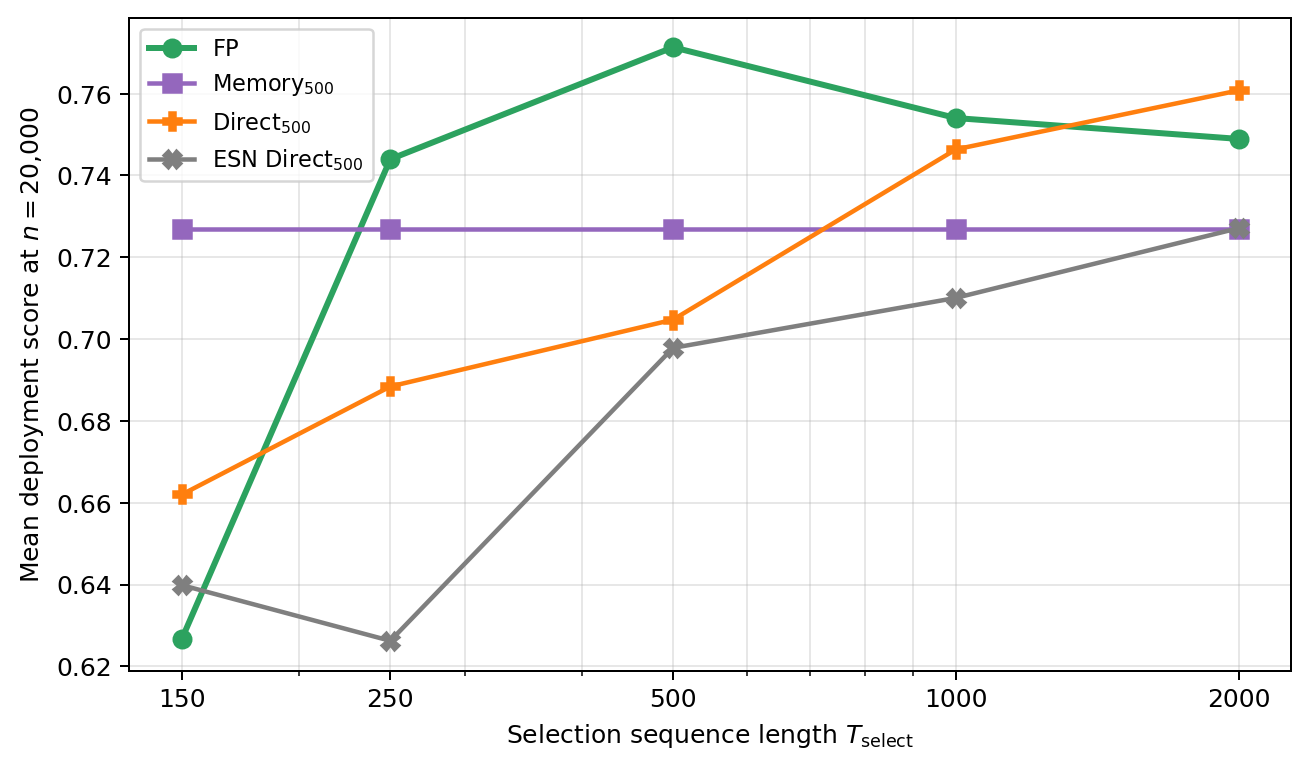}
  \caption{Mean deployment score at $n=20\,000$ as a function of
  $T_{\rm select}$. Higher is better. FP reaches its best observed
  mean at $T_{\rm select}=500$; Direct$_{500}$ and ESN Direct$_{500}$
  improve with longer sequences. Memory$_{500}$ selects the same
  candidate across this sweep and is therefore flat; its proxy uses
  $D=\min\{T_{\rm select},1000\}$.}
  \label{fig:timing-perf}
\end{minipage}
\hfill
\begin{minipage}[t]{0.48\textwidth}
  \vspace{0pt}
  \centering
  \includegraphics[width=\linewidth]{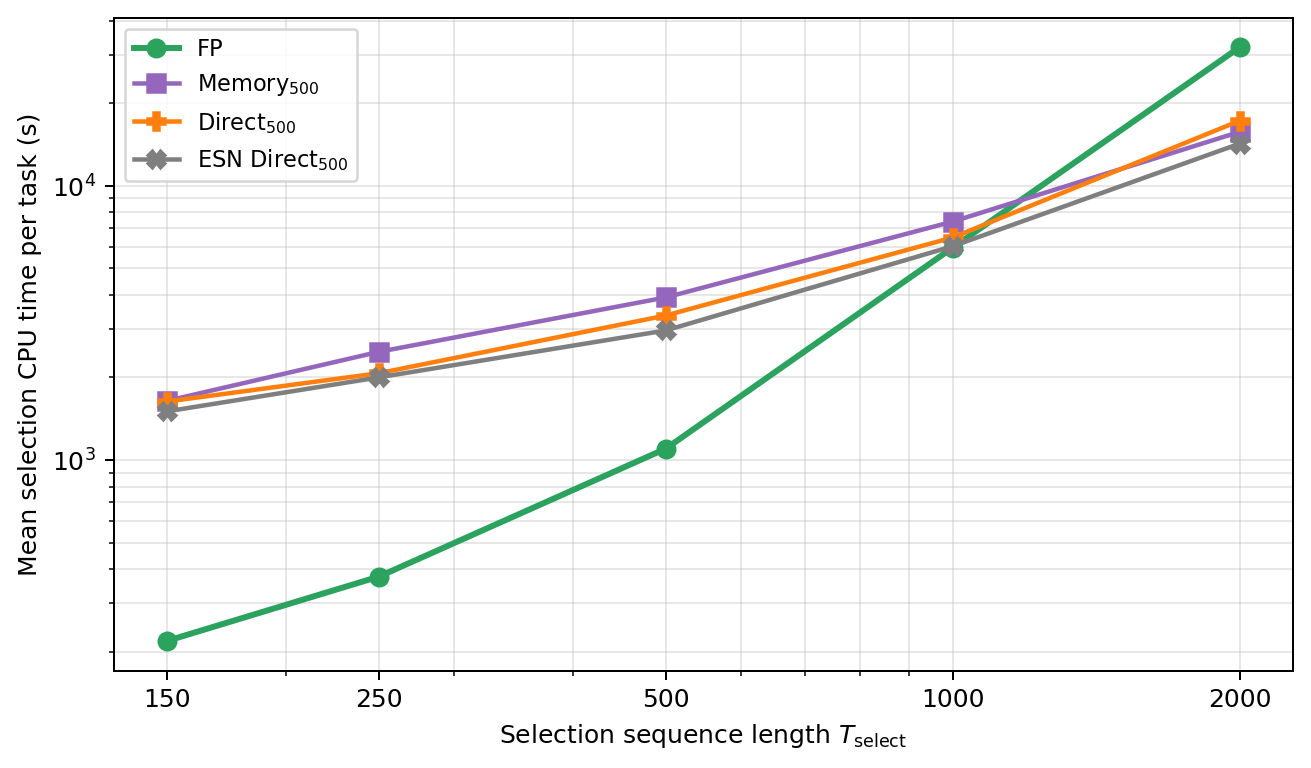}
  \caption{Mean selection CPU time per task averaged over the ten-task
  synthetic suite as a function of $T_{\rm select}$
  ($n_{\rm select}=500$). All points are measured runs. FP uses the
  closed-form erf-surrogate selection kernel. Lower is better.}
  \label{fig:timing-avg}
\end{minipage}
\end{figure*}

\subsection{Surrogate-kernel validation}
\label{app:surrogate-validation}

The main synthetic selection and timing experiments use the closed-form
erf surrogate of Remark~\ref{rem:erf-surrogate}. All real-data
forecasting experiments, both ETT and Telco, use the exact $\tanh$
kernel evaluated by Gauss--Hermite quadrature. The surrogate is
therefore used only as a computational accelerator in the controlled
synthetic setting.

We validated this choice by comparing matched erf-surrogate and
exact-$\tanh$ FP selections on the synthetic suite. Across the checked
selection lengths, the largest task-level deployment-score difference
was $6.3\times 10^{-3}$, and the mean difference was $O(10^{-3})$,
well below the task-level variability in the deployment tables. This
supports using the surrogate for the synthetic selection and timing
sweeps. By contrast, the real-data experiments are reported with the
exact $\tanh$ kernel, so no real-data conclusion relies on the
surrogate approximation.

\section{Forecasting Horizon-Level Results}
\label{sec:forecasting-horizons}

Table~\ref{tab:forecasting-horizons} gives the underlying
horizon-level scores.

\begin{table*}[t]
\centering
\small
\setlength{\tabcolsep}{4pt}
\caption{Horizon-level forecasting results at $n=20\,000$. Scores are
$1-\NRMSE$; higher is better. ETTh is sampled hourly, so horizon steps
coincide with physical hours; ETTm is sampled every 15 minutes, and
horizons are given in steps with the physical horizon in parentheses.
ETT deviations are over three deployment seeds for deterministic
selectors; Telco deviations are across ten cells. Telco selection
optimizes the aggregate over horizons 1--12; the table shows three
representative horizons, so their mean differs from the aggregate scores
of Table IV of the main paper.}
\label{tab:forecasting-horizons}
\begin{tabular}{llcccc}
\toprule
Dataset & Metric & FP & Memory$_{500}$ & ESN Direct$_{500}$ & Direct$_{500}$ \\
\midrule
ETTh1 & 1 h  & $0.695\pm0.00056$ & $0.597\pm0.00024$ & $0.731\pm0.00048$ & $0.724\pm0.00068$ \\
      & 6 h  & $0.587\pm0.00012$ & $0.576\pm0.00024$ & $0.586\pm0.00156$ & $0.589\pm0.00044$ \\
      & 12 h & $0.386\pm0.00102$ & $0.392\pm0.00009$ & $0.374\pm0.00213$ & $0.387\pm0.00100$ \\
\midrule
ETTh2 & 1 h  & $0.836\pm0.00022$ & $0.583\pm0.00010$ & $0.835\pm0.00063$ & $0.836\pm0.00022$ \\
      & 6 h  & $0.620\pm0.00135$ & $0.509\pm0.00005$ & $0.609\pm0.00165$ & $0.620\pm0.00135$ \\
      & 12 h & $0.537\pm0.00027$ & $0.483\pm0.00013$ & $0.542\pm0.00144$ & $0.537\pm0.00027$ \\
\midrule
ETTm1 & 4 steps (1 h)   & $0.719\pm0.00145$ & $0.702\pm0.00021$ & $0.728\pm0.00059$ & $0.714\pm0.00088$ \\
      & 24 steps (6 h)  & $0.446\pm0.00116$ & $0.435\pm0.00012$ & $0.449\pm0.00079$ & $0.438\pm0.00191$ \\
      & 48 steps (12 h) & $0.351\pm0.00123$ & $0.357\pm0.00003$ & $0.346\pm0.00097$ & $0.345\pm0.00089$ \\
\midrule
ETTm2 & 4 steps (1 h)   & $0.862\pm0.00046$ & $0.762\pm0.00008$ & $0.861\pm0.00025$ & $0.862\pm0.00046$ \\
      & 24 steps (6 h)  & $0.456\pm0.00020$ & $0.474\pm0.00014$ & $0.457\pm0.00013$ & $0.456\pm0.00020$ \\
      & 48 steps (12 h) & $0.442\pm0.00004$ & $0.501\pm0.00025$ & $0.443\pm0.00012$ & $0.442\pm0.00004$ \\
\midrule
Telco & 1 h  & $0.601\pm0.180$ & $0.679\pm0.092$ & $0.644\pm0.151$ & $0.685\pm0.088$ \\
      & 6 h  & $0.567\pm0.162$ & $0.590\pm0.130$ & $0.593\pm0.134$ & $0.610\pm0.119$ \\
      & 12 h & $0.549\pm0.192$ & $0.589\pm0.128$ & $0.585\pm0.145$ & $0.611\pm0.119$ \\
\bottomrule
\end{tabular}
\end{table*}

\bibliographystyle{IEEEtran}
\bibliography{fp_rc}

@inproceedings{couillet2016random,
  author    = {Romain Couillet and Gilles Wainrib and Hafiz Tiomoko Ali and Harry Sevi},
  title     = {A random matrix approach to echo-state neural networks},
  booktitle = {International Conference on Machine Learning (ICML)},
  pages     = {517--525},
  year      = {2016}
}

@article{gonon2025reservoir,
  author  = {Lukas Gonon and Lyudmila Grigoryeva and Juan-Pablo Ortega},
  title   = {Reservoir Kernels and Volterra Series},
  journal = {IEEE Transactions on Neural Networks and Learning Systems},
  volume  = {37},
  number  = {5},
  pages   = {2181--2192},
  year    = {2026},
  doi     = {10.1109/TNNLS.2025.3630143}
}

@article{grigoryeva2018universal,
  author  = {Lyudmila Grigoryeva and Juan-Pablo Ortega},
  title   = {Universal discrete-time reservoir computers with stochastic inputs and linear reservoirs using non-homogeneous state-affine systems},
  journal = {Journal of Machine Learning Research},
  volume  = {19},
  number  = {24},
  pages   = {1--40},
  year    = {2018}
}

@inproceedings{
gu2021efficiently,
title={{Efficiently Modeling Long Sequences with Structured State Spaces}},
author={Albert Gu and Karan Goel and Christopher Re},
booktitle={International Conference on Learning Representations},
year={2022},
url={https://openreview.net/forum?id=uYLFoz1vlAC}
}

@article{jaeger2001echo,
  title={The “echo state” approach to analysing and training recurrent neural networks-with an erratum note},
  author={Jaeger, Herbert},
  journal={Bonn, Germany: German national research center for information technology gmd technical report},
  volume={148},
  number={34},
  pages={13},
  year={2001},
  publisher={Bonn}
}

@inproceedings{lagomarsini2025benchmarking,
  title={Benchmarking Nonlinear Readouts in Linear Reservoir Networks},
  author={Lagomarsini, Giacomo and Ceni, Andrea and Gallicchio, Claudio},
  booktitle={International Conference on Artificial Neural Networks},
  pages={176--187},
  year={2025},
  organization={Springer}
}

@article{lukosevicius2009reservoir,
  author  = {Mantas Luko{\v{s}}evi{\v{c}}ius and Herbert Jaeger},
  title   = {Reservoir computing approaches to recurrent neural network training},
  journal = {Computer Science Review},
  volume  = {3},
  number  = {3},
  pages   = {127--149},
  year    = {2009}
}

@article{tanaka2019,
  title={Recent advances in physical reservoir computing: A review},
  author={Tanaka, Gouhei and Yamane, Toshiyuki and H{\'e}roux, Jean Benoit and Nakane, Ryosho and Kanazawa, Naoki and Takeda, Seiji and Numata, Hidetoshi and Nakano, Daiju and Hirose, Akira},
  journal={Neural Networks},
  volume={115},
  pages={100--123},
  year={2019},
  publisher={Elsevier}
}

@inproceedings{orvieto2023resurrecting,
  title={Resurrecting recurrent neural networks for long sequences},
  author={Orvieto, Antonio and Smith, Samuel L and Gu, Albert and Fernando, Anushan and Gulcehre, Caglar and Pascanu, Razvan and De, Soham},
  booktitle={International conference on machine learning},
  pages={26670--26698},
  year={2023},
  organization={PMLR}
}

@book{mingo2017,
  title={Free probability and random matrices},
  author={Mingo, James A and Speicher, Roland},
  volume={35},
  year={2017},
  publisher={Springer}
}

@article{BordenaveChafai2012,
  author  = {Bordenave, Charles and Chafa{\"i}, Djalil},
  title   = {Around the Circular Law},
  journal = {Probability Surveys},
  volume  = {9},
  pages   = {1--89},
  year    = {2012},
  doi     = {10.1214/11-PS183}
}

@article{CollinsMale2014,
  author  = {Collins, Beno{\^i}t and Male, Camille},
  title   = {The Strong Asymptotic Freeness of Haar and Deterministic Matrices},
  journal = {Annales Scientifiques de l'\'{E}cole Normale Sup\'{e}rieure},
  series  = {4},
  volume  = {47},
  number  = {1},
  pages   = {147--163},
  year    = {2014},
  doi     = {10.24033/asens.2211}
}

@article{schultz2005,
  author  = {Hanne Schultz},
  title   = {Non-commutative polynomials of independent {Gaussian}
             random matrices: The real and symplectic cases},
  journal = {Probability Theory and Related Fields},
  volume  = {131},
  number  = {2},
  pages   = {261--309},
  year    = {2005},
  doi     = {10.1007/s00440-004-0366-7}
}

@inproceedings{dong2020reservoir,
  author    = {Dong, Jonathan and Ohana, Ruben and Rafayelyan, Mushegh and Krzakala, Florent},
  title     = {Reservoir computing meets recurrent kernels and structured transforms},
  booktitle = {Advances in Neural Information Processing Systems},
  volume    = {33},
  pages     = {16785--16796},
  year      = {2020}
}

@article{bergstra2012random,
  author  = {Bergstra, James and Bengio, Yoshua},
  title   = {Random search for hyper-parameter optimization},
  journal = {Journal of Machine Learning Research},
  volume  = {13},
  pages   = {281--305},
  year    = {2012}
}

@techreport{jaeger2001memory,
  author      = {Jaeger, Herbert},
  title       = {Short term memory in echo state networks},
  institution = {German National Research Center for Information Technology (GMD)},
  number      = {GMD Report 152},
  year        = {2001},
  address     = {Sankt Augustin, Germany}
}

@book{vershynin2018high,
  author    = {Vershynin, Roman},
  title     = {High-Dimensional Probability: An Introduction with Applications in Data Science},
  publisher = {Cambridge University Press},
  series    = {Cambridge Series in Statistical and Probabilistic Mathematics},
  volume    = {47},
  year      = {2018}
}

@book{bhatia1997matrix,
  author    = {Bhatia, Rajendra},
  title     = {Matrix Analysis},
  publisher = {Springer},
  series    = {Graduate Texts in Mathematics},
  volume    = {169},
  year      = {1997}
}

@inproceedings{bergstra2011algorithms,
  title     = {Algorithms for Hyper-Parameter Optimization},
  author    = {Bergstra, James and Bardenet, R{\'e}mi
               and Bengio, Yoshua and K{\'e}gl, Bal{\'a}zs},
  booktitle = {Advances in Neural Information Processing Systems},
  volume    = {24},
  pages     = {2546--2554},
  year      = {2011},
  publisher = {Curran Associates, Inc.}
}

@incollection{lukosevicius2012practical,
  author    = {Luko{\v{s}}evi{\v{c}}ius, Mantas},
  title     = {A practical guide to applying echo state networks},
  booktitle = {Neural Networks: Tricks of the Trade},
  edition   = {2},
  publisher = {Springer},
  pages     = {659--686},
  year      = {2012}
}

@article{schrauwen2008improving,
  author  = {Schrauwen, Benjamin and Wardermann, Marion and Verstraeten, David and Steil, Jochen J. and Stroobandt, Dirk},
  title   = {Improving reservoirs using intrinsic plasticity},
  journal = {Neurocomputing},
  volume  = {71},
  number  = {7--9},
  pages   = {1159--1171},
  year    = {2008}
}

@article{yildiz2012revisiting,
  title={Re-visiting the echo state property},
  author={Yildiz, Izzet B and Jaeger, Herbert and Kiebel, Stefan J},
  journal={Neural networks},
  volume={35},
  pages={1--9},
  year={2012},
  publisher={Elsevier}
}

@article{hermans2012recurrent,
  title={Recurrent kernel machines: Computing with infinite echo state networks},
  author={Hermans, Michiel and Schrauwen, Benjamin},
  journal={Neural Computation},
  volume={24},
  number={1},
  pages={104--133},
  year={2012},
  publisher={MIT Press}
}

@article{d2025comparison,
  title={Comparison of reservoir computing topologies using the recurrent kernel approach},
  author={D’Inverno, Giuseppe Alessio and Dong, Jonathan},
  journal={Neurocomputing},
  volume={611},
  pages={128679},
  year={2025},
  publisher={Elsevier}
}

@article{thiede2019gradient,
  title={Gradient based hyperparameter optimization in echo state networks},
  author={Thiede, Luca Anthony and Parlitz, Ulrich},
  journal={Neural Networks},
  volume={115},
  pages={23--29},
  year={2019},
  publisher={Elsevier}
}

@article{racca2021robust,
  title={Robust optimization and validation of echo state networks for learning chaotic dynamics},
  author={Racca, Alberto and Magri, Luca},
  journal={Neural Networks},
  volume={142},
  pages={252--268},
  year={2021},
  publisher={Elsevier}
}

@inproceedings{lukosevicius2019efficient,
  title={Efficient cross-validation of echo state networks},
  author={Luko{\v{s}}evi{\v{c}}ius, Mantas and Uselis, Arnas},
  booktitle={International conference on artificial neural networks},
  pages={121--133},
  year={2019},
  organization={Springer}
}

@inproceedings{maat2018efficient,
  title={Efficient optimization of echo state networks for time series datasets},
  author={Maat, Jacob Reinier and Gianniotis, Nikos and Protopapas, Pavlos},
  booktitle={2018 International Joint Conference on Neural Networks (IJCNN)},
  pages={1--7},
  year={2018},
  organization={IEEE}
}

@article{gonon2020universality,
  author  = {Lukas Gonon and Juan-Pablo Ortega},
  title   = {Reservoir Computing Universality With Stochastic Inputs},
  journal = {IEEE Transactions on Neural Networks and Learning Systems},
  volume  = {31},
  number  = {1},
  pages   = {100--112},
  year    = {2020},
  doi     = {10.1109/TNNLS.2019.2899649}
}

@inproceedings{matzner2022hyperparameter,
  title={Hyperparameter tuning in echo state networks},
  author={Matzner, Filip},
  booktitle={Proceedings of the Genetic and Evolutionary Computation Conference},
  pages={404--412},
  year={2022}
}

@inproceedings{verstraeten2009quantification,
  title={On the quantification of dynamics in reservoir computing},
  author={Verstraeten, David and Schrauwen, Benjamin},
  booktitle={International Conference on Artificial Neural Networks},
  pages={985--994},
  year={2009},
  organization={Springer}
}

@inproceedings{zhou2021informer,
  title={Informer: Beyond efficient transformer for long sequence time-series forecasting},
  author={Zhou, Haoyi and Zhang, Shanghang and Peng, Jieqi and Zhang, Shuai and Li, Jianxin and Xiong, Hui and Zhang, Wancai},
  booktitle={Proceedings of the AAAI conference on artificial intelligence},
  volume={35},
  number={12},
  pages={11106--11115},
  year={2021}
}

@article{atiya2000new,
  title={New results on recurrent network training: unifying the algorithms and accelerating convergence},
  author={Atiya, Amir F and Parlos, Alexander G},
  journal={IEEE transactions on neural networks},
  volume={11},
  number={3},
  pages={697--709},
  year={2000},
  publisher={IEEE}
}

@article{inubushi2017reservoir,
  title={Reservoir computing beyond memory-nonlinearity trade-off},
  author={Inubushi, Masanobu and Yoshimura, Kazuyuki},
  journal={Scientific reports},
  volume={7},
  number={1},
  pages={10199},
  year={2017},
  publisher={Nature Publishing Group UK London}
}

@article{lorentz1963deterministic,
  title={Deterministic non-periodic flow},
  author={Edward N. Lorenz},
  journal={J. Atmos. Sci.},
  volume={20},
  pages={130--141},
  year={1963}
}

@article{mackey1977oscillation,
  title={Oscillation and chaos in physiological control systems},
  author={Mackey, Michael C and Glass, Leon},
  journal={Science},
  volume={197},
  number={4300},
  pages={287--289},
  year={1977},
  publisher={American Association for the Advancement of Science}
}

@inproceedings{lorenz1996predictability,
  author    = {Edward N. Lorenz},
  title     = {Predictability: A Problem Partly Solved},
  booktitle = {Proceedings of the ECMWF Seminar on Predictability},
  volume    = {1},
  address   = {Reading, UK},
  pages     = {1--18},
  year      = {1996}
}

\end{document}